\documentclass{article}

\usepackage{microtype}
\usepackage[pdftex]{graphicx}
\usepackage{subfigure}
\usepackage{booktabs} 

\usepackage{hyperref}

\usepackage[accepted]{icml2023}

\usepackage{amsmath}
\usepackage{amssymb}
\usepackage{mathtools}
\usepackage{amsthm}

\usepackage[capitalize,noabbrev]{cleveref}

\theoremstyle{plain}
\newtheorem{theorem}{Theorem}[section]

\newtheorem{lemma}[theorem]{Lemma}

\theoremstyle{definition}
\newtheorem{definition}[theorem]{Definition}
\newtheorem{assumption}[theorem]{Assumption}
\theoremstyle{remark}
\newtheorem{remark}[theorem]{Remark}

\usepackage[textsize=tiny]{todonotes}

\icmltitlerunning{Approximation and Estimation Ability of Transformers for Sequence-to-Sequence Functions with Infinite Dimensional Input}

\usepackage{physics}

\newcommand{\softmax}{\operatorname{Softmax}}
\newcommand{\clip}{\operatorname{clip}}
\newcommand{\enc}{\operatorname{Enc}}
\newcommand{\domain}{[0, 1]^{d\times \infty}}
\newcommand{\R}{\mathbb{R}}
\newcommand{\Z}{\mathbb{Z}}
\newcommand{\N}{\mathbb{N}}
\newcommand{\floor}[1]{\lfloor #1 \rfloor}

\newcommand{\supp}{\operatorname{supp}}
\DeclareMathAlphabet{\mymathbb}{U}{BOONDOX-ds}{m}{n}
\newcommand{\poly}{\operatorname{poly}}
\newcommand{\close}[1]{\overset{#1}{\eqsim}}

\begin{document}

\twocolumn[
	\icmltitle{Approximation and Estimation Ability of Transformers \\
		for Sequence-to-Sequence Functions with Infinite Dimensional Input}



	\icmlsetsymbol{equal}{*}

	\begin{icmlauthorlist}
		\icmlauthor{Shokichi Takakura}{todai,riken}
		\icmlauthor{Taiji Suzuki}{todai,riken}
	\end{icmlauthorlist}

	\icmlaffiliation{todai}{Department of Mathematical Informatics, the University of Tokyo, Tokyo, Japan}
	\icmlaffiliation{riken}{Center for Advanced Intelligence Project,
RIKEN, Tokyo, Japan}

	\icmlcorrespondingauthor{Shokichi Takakura}{masayoshi361@g.ecc.u-tokyo.ac.jp}

	\icmlkeywords{Machine Learning, Transformer, non-parametric regression}

	\vskip 0.3in
]



\printAffiliationsAndNotice{} 

\begin{abstract}
	Despite the great success of Transformer networks in various applications such as natural language processing and computer vision, their theoretical aspects are not well understood.
	In this paper, we study the approximation and estimation ability of Transformers as sequence-to-sequence functions with infinite dimensional inputs.
	Although inputs and outputs are both infinite dimensional, we show that when the target function has anisotropic smoothness,
	Transformers can avoid the curse of dimensionality due to their feature extraction ability and parameter sharing property.
	In addition, we show that even if the smoothness changes depending on each input,
	Transformers can estimate the importance of features for each input and extract important features dynamically.
	Then, we proved that Transformers achieve similar convergence rate as in the case of the fixed smoothness.
	Our theoretical results support the practical success of Transformers for high dimensional data.
\end{abstract}

\section{Introduction}
	Transformer networks, first proposed in~\citet{vaswani_attention_2017}, empirically show high performance in various fields including natural language processing~\citep{vaswani_attention_2017}, computer vision~\citep{dosovitskiy_image_2021} and audio processing~\citep{dong_speech-transformer_2018}, where the dimensionality of inputs is relatively high.
	However, despite the growing interest in Transformer models, their theoretical properties are still unclear.

	Aside from the Transformer architecture, there exists a line of work which studied the approximation and estimation ability of
	fully connected neural networks (FNN) for certain function spaces such as H\"{o}lder spaces~\citep{schmidt-hieber_nonparametric_2020} and Besov spaces~\citep{suzuki_adaptivity_2019}.
	For example,~\citet{schmidt-hieber_nonparametric_2020} showed that FNNs with ReLU activation can achieve
	the near minimax optimal rate of the estimation error for composite functions in H\"{o}lder spaces.
	Although deep learning can achieve the near optimal rate for several function classes,
	the convergence rate is often strongly affected by the dimensionality of inputs.
	Some researches \citep{nakada_adaptive_2022,chen_nonparametric_2022} considered the settings where the data are distributed on a low dimensional manifold
	and showed that deep neural networks can avoid the curse of dimensionality.
	However, this assumption is relatively strong
	since the low dimensionality of the data manifold is easily destroyed by noise injection.
	Then,~\citet{suzuki_adaptivity_2019} showed that even if the data manifold is not low dimensional,
	deep neural networks can avoid the curse of dimensionality
	under the assumption that the target function has anisotropic smoothness.
	Moreover,~\citet{okumoto_learnability_2022} showed that (dilated) convolutional neural networks (CNN) can avoid the curse of dimensionality
	even if inputs are infinite dimensional.

	Although the learnability of FNNs and CNNs has been intensively studied,
	that of Transformer networks is not well understood.
	There are some researches \citep{edelman_inductive_2022,gurevych_rate_2022} which studied the learning ability of Transformers.
	\citet{edelman_inductive_2022} evaluated the capacity of Transformer networks and derived the sample complexity to learn sparse Boolean functions.
	Since they investigated discrete inputs, the smoothness of the target function was not considered.
	\citet{gurevych_rate_2022} studied binary classification tasks and proved that Transformer networks can avoid the curse of dimensionality when a posteriori probability is represented by the hierarchical composition model with H\"{o}lder smoothness.
	However, these studies have some limitations.
	First, the previous works considered the fixed length input, although Transformer networks can be applied to sequences of any length due to the parameter sharing property, even if the length is infinite.
	Indeed, Transformers are often applied to very high dimensional data such as images and languages,
	and the learnability of Transformers for such extremely high dimensional data is still unclear.
	Second, their analysis is limited to the single output setting.
	In some applications such as question-answering, it is necessary to learn a function that maps an input sequence to an output sequence.
	Transformer networks can be applied to such situations and achieve great practical success as represented by BERT~\citep{devlin_bert_2019}
	although output is also high dimensional.
	Finally, in these studies, the intrinsic structure of target functions does not depend on each input and the pattern of attention weights does not change.
	This is in contrast to the dynamic nature of self-attention matrices observed in practice~\citep{likhosherstov_expressive_2021}.

	In this paper, we consider the non-parametric regression problems and study the approximation and estimation ability of Transformers for \textit{sequence-to-sequence} functions with \textit{infinite} dimensional inputs. 
        In the high dimensional setting, the dependence of the target function on inputs varies depending on the direction.
        For example, in image classification, the target function is more dependent on foreground features than background features.
        To deal with such situations, we consider \textit{direction-dependent} smoothness.
	Then, we derive the convergence rate of errors for mixed and anisotropic smooth functions
	and show that Transformer networks can avoid the curse of dimensionality.
	In addition, we consider the setting where the position of important features changes depending on each input
	and show that Transformer networks can avoid the curse of dimensionality, which reveals the \textit{dynamical} feature extraction ability of self-attention mechanism.
	Our contribution can be summarized as follows.
	\begin{itemize}
		\item We derive the convergence rate of approximation and estimation error for shift-equivariant functions with the mixed or anisotropic smoothness.
			We show that the errors are dependent only on the smoothness of the target function and independent of the input and output dimension.
			This means that Transformers can avoid the curse of dimensionality even if the dimensionality of inputs and outputs is infinite.
		\item We consider the situation where the smoothness of each coordinate, which corresponds to the importance of each feature, changes depending on inputs and derive the similar convergence rate to the case of the fixed smoothness.
	\end{itemize}
	\subsection{Other Related Works}
		\citet{yun_are_2020, zaheer_big_2020} proved that Transformers with learnable positional encodings are universal approximators of
		continuous sequence-to-sequence functions with compact support,
		but the results suffer from the curse of dimensionality. This is unavoidable as mentioned in~\citet{yun_are_2020}.
		To derive meaningful convergence results, it is necessary to restrict the function class.
		From this perspective,~\citet{edelman_inductive_2022} investigated sparse Boolean functions
		and~\citet{gurevych_rate_2022} studied the hierarchical composition model.
		However, our analysis imposes smoothness structure on target functions more directly compared to these studies.

		The mixed and anisotropic smooth function spaces which we consider in this study are extensions of the functions investigated in~\citet{okumoto_learnability_2022}.
		In addition, the function spaces can be seen as an infinite dimensional counterpart of the mixed Besov space~\citep{schmeisser_unconditional_1987} and anisotropic Besov space~\citep{nikol_approximation_1975}.
		From the deep learning perspective, the approximation and estimation error of FNNs for the mixed Besov space and anisotropic Besov space was analyzed in~\citet{suzuki_adaptivity_2019} and~\citet{suzuki_deep_2021}, respectively.
		However, it is not trivial to extend these results to Transformer architecture and multiple output setting.

		In this paper, we also consider the piecewise $\gamma$-smooth function class.
		This function class is inspired by the piecewise smooth functions, which was investigated in~\citet{petersen_optimal_2018, imaizumi_deep_2019},
		but these studies did not consider anisotropic smoothness and Transformer networks.

		There are some studies that investigated the theoretical properties of Transformer networks from different perspective than ours.
		\citet{jelassi_vision_2022} analyzed simplified Vision transformers and showed that they can learn the spatial structure via gradient descent.
		\citet{zhang_analysis_2022} studied the self-attention mechanism from the perspective of exchangeability and proved that Transformer networks can learn desirable representation of input tokens.
		\citet{perez_turing_2019} showed the Turing completeness of Transformers
		and~\citet{wei_statistically_2021} introduced the notion of \textit{statistically meaningful approximation} and gave the sample complexity to approximate Boolean circuits and Turing machines.
		\citet{likhosherstov_expressive_2021} showed that a self-attention module with fixed parameters can approximate any sparse matrix by designing an input appropriately.
	\subsection{Notations}
		Here, we prepare the notations. 
		For $l \in \N$, let $[l]$ be the set $\qty{1, \dots, l}$ and for $l, r \in \Z~(l \leq r)$, let $[l:r]$ be the set $\qty{l, \dots, r}$.
		For a set $\mathbb{S} \subset \R$ and $d \in \N$, let
		\begin{align*}
                \mathbb{S}^{d\times \infty}   & := \qty{[\dots, s_{-1}, s_0, s_1, \dots, s_i, \dots] \mid s_i \in \mathbb{S}^d}, \\
			\mathbb{S}_0^{d\times \infty} & := \qty{s \in (\mathbb{S}\cup \qty{0})^{d\times \infty} \mid \abs{\supp(s)} < \infty},
		\end{align*}
		where $\supp(s)$ is defined as $\qty{(i, j) \in [d] \times \Z \mid s_{i, j} \neq 0}$.
            Similary, $\mathbb{S}^{d\times [l:r]}$ denotes the set $\qty{[s_l, \dots s_r] \mid s_j \in \mathbb{S}}$.
            For $X = [\dots, x_0, x_1, \dots] \in \R^{d\times \infty}$, $X[l:r]$ denotes $[x_l, \dots, x_r] \in \R^{d\times [l:r]}$.
		For $s \in \R_0^{d\times \infty}$, let $2^s := 2^{\sum_{i \in [d], j \in \Z}s_{ij}}$.
		For $X \in \R^{d \times \infty}$, $\norm{X}_\infty$ denotes $\sup_{i \in [d], j \in \Z}\abs{X_{i, j}}$ and for $x \in \R^l$, $\norm{x}_1$ denotes $\sum_{i=1}^l \abs{x_i}$.
		For $F:\Omega \to \R^{l}$, let $\norm{F}_\infty := \sup_{X\in \Omega} \norm{F(X)}_\infty$.
		For the probability measure $P_X$ on $\Omega$ and $p > 0$, the norm $\norm{\cdot}_{p, P_X}$ is defined by
		\begin{align*}
			\norm{f}_{p, P_X} & = \qty(\int_{\Omega} \norm{f(X)}_p^p \dd{P_X})^{1/p}.
		\end{align*}
		For a matrix $A$, let $\norm{A}_0 = \abs{\qty{(i, j) \mid A_{ij} \neq 0}}$.
		For $j \in \Z$, we define the shift operator $\Sigma_j: \R^{d\times \infty} \to \R^{d\times \infty}$ by $(\Sigma_j(X))_i = x_{i + j}$ for $X = [\dots, x_0, \dots, x_i, \dots] \in \R^{d\times \infty}$.
		For a normed space $\mathcal{F}$, we define $U(\mathcal{F})$ by $U(\mathcal{F}) := \qty{f \in \mathcal{F} \mid \norm{f}_{\mathcal{F}} \leq 1}$, where $\norm{\cdot}_\mathcal{F}$ is the norm of $\mathcal{F}$.
\section{Problem Settings}
	\subsection{Non-parametric Regression Problems}
		In this paper, we consider non-parametric regression problems with infinite dimensional inputs.
		We regard an input $X \in \domain$ as a bidirectional sequence of tokens $\qty{x_i}_{i=-\infty}^\infty ~(x_i \in \R^d)$.
		For example, each token $x_i$ corresponds to a word vector in natural language processing and an image patch in image processing~\citep{dosovitskiy_image_2021}.
		Let $P_X$ be a probability measure on $(\domain, \mathcal{B}(\domain))$. We write $\Omega$ for the support of $P_X$.
		We assume that $P_X$ is shift-invariant. That is, for any $i \in \Z$ and $B \in \mathcal{B}(\domain)$, $P_X(B) = P_X(\qty{\Sigma_i(X) \mid X \in B})$.
		In the non-parametric regression,
		we observe $n$ i.i.d. pairs of inputs $X^{(i)} \sim P_X$ and outputs $Y^{(i)} \in \R^\infty$.
		We assume that there exists a true function $F^\circ : \Omega \to \R^\infty$, and outputs $Y^{(i)}$ is given by
		\begin{align*}
			Y^{(i)} & := F^\circ (X^{(i)}) + \xi^{(i)},
		\end{align*}
		where the noise $\xi^{(i)}_j$ follows the normal distribution $N(0, \sigma^2)~(\sigma > 0)$ independently.
		We also assume that $\qty{\xi^{(i)}}_{i=1}^n$ are independent of $\qty{X^{(i)}}_{i=1}^n$.
		Note that unlike~\citet{okumoto_learnability_2022}, we do \textit{not} assume that the Radon-Nikodym derivative $\frac{\dd{P_X}}{\dd{\lambda}}$ for the uniform distribution $\lambda$ on $(\domain, \mathcal{B}(\domain))$ satisfies $\norm{\frac{\dd{P_X}}{\dd{\lambda}}}_\infty < \infty$.

		Based on the observed data $\mathcal{D}^n:=\qty{(X^{(i)}, Y^{(i)})}_{i=1}^n$,
		we compute an estimator $\hat F$ which takes its value in the class of Transformer networks.
		To evaluate the statistical performance of an estimator $\hat F$, we consider the mean squared error
		\begin{align*}
			R_{l, r}(\hat F, F^\circ) = \frac{1}{r - l + 1}\sum_{i=l}^r \mathbb{E}\qty[\norm{\hat F_i - F^\circ_i}_{2, P_X}^2],
		\end{align*}
		where the expectation is taken with respect to the training data $\mathcal{D}^n$.
		Here, to avoid the convergence argument, we consider a finite number of outputs ($Y^{(i)}_{l}, \dots, Y^{(i)}_r$),
		but we show later that the convergence rate of estimation error does not depend on $l$ and $r$.

		In this paper, we consider an empirical risk minimization (ERM) estimator, which is defined as a minimizer of the following minimization problem:
		\begin{align*}
			\min_{F \in \mathcal{T}} \sum_{i=1}^n \sum_{j=l}^r \qty(F(X^{(i)})_j - Y^{(i)}_j)^2,
		\end{align*}
		where $\mathcal{T}$ is supposed to be the set of Transformer networks defined in Eq.~\eqref{eq:transformer-networks}.
		Note that an ERM estimator $\hat F$ is a random variable which depends on the training dataset $\mathcal{D}^n$.
		In practice, it is difficult to solve the problem due to the non-convexity of the objective function.
		Some studies~\citep{huang_improving_2020, jelassi_vision_2022} investigated the optimization aspect of Transformers,
		but we do not pursue this direction in this study.

	\subsection{Transformer Architecture}
		Transformer architecture has three main components:
		\vspace{-0.8\baselineskip}
		\begin{enumerate}
			\renewcommand{\labelenumi}{(\roman{enumi})}
			\setlength{\itemsep}{0mm} 
			\item (position-wise) FNN layer.
			\item Self-attention layer.
			\item Embedding layer.
		\end{enumerate}

		(i) First, we introduce FNN layers.
		An FNN with depth $L$ and width $W$ is defined as
		\begin{align*}
			f(x) := (A_L\eta(\cdot) + b_L) \circ \dots \circ (A_1x + b_1),
		\end{align*}
		where $A_i \in \R^{d_{i + 1} \times d_i}, b_i \in \R^{d_{i+1}}$, $\max_i d_i \leq W$, and ReLU activation function $\eta(x) = \max\qty{x, 0}$ is operated in an element-wise manner.
		Then, we define the class of FNN with depth $L$, width $W$, norm bound $B$ and sparsity $S$ by
		\begin{align*}
			\Psi(L, W, S, B) := & \left\{f \mid \max_i \qty{\norm{A_i}_\infty, \norm{b_i}_\infty} \leq B,\right. \\
			                    & \quad \left. \sum_{i=1}^L \norm{A_i}_0 + \norm{b_i}_0 \leq S \right\}.
		\end{align*}

		(ii) Next, we define the self-attention layer. In this paper, we consider the sliding window attention,
		which is used in some practical architectures such as Longformer~\citep{beltagy_longformer_2020} and Big Bird~\citep{zaheer_big_2020}.
		To focus on local context, the sliding window attention restricts the receptive field to the input around each token.
		Let $D$ be the embedding dimension, $H$ be the number of head, and $U$ be the window size.
		Then, self-attention layer $g$ with parameters $K_h \in \R^{D' \times D}, Q_h \in \R^{D' \times D}, V_h \in \R^{D\times D}~(D' \leq D, h=1, \dots, H)$ is defined by
		\begin{align*}
			g(X)_i & := x_i + \sum_{h=1}^H V_h X[i-U:i+U] A_h,
		\end{align*}
		where
		\begin{align*}
			A_h & := \softmax((K^{(h)}X[i-U:i+U])^\top(Q^{(h)}x_i)), \\
			    & \in \R^{[i-U:i+U]}.
		\end{align*}
		Here, $\softmax: \R^l \to \R^l$ is defined by
		\begin{align*}
			\softmax(x) = \qty[\frac{e^{x_{1}}}{\sum\limits_{j \in [l]} e^{x_{j}}}, \dots, \frac{e^{x_{l}}}{\sum\limits_{j \in [l]} e^{x_{j}}}]^\top.
		\end{align*}
		Then, we define the class of self-attention layers with the window size $U \in \N$, the embedding dimension $D$, the number of head $H$, and the norm bound $B$ by
		\begin{align*}
			 & \mathcal{A}(U, D, H, B)                                                                           \\
			 & \quad := \qty{g \mid \max_h \qty{\norm*{K_h}_\infty,\norm*{Q_h}_\infty, \norm*{V_h}_\infty} \leq B}.
		\end{align*}

		(iii) Finally, we define the embedding layer.
		For embedding dimension $D$, an embedding layer is defined as
		\begin{align*}
			\enc_P(X) & = EX + P,
		\end{align*}
		where $E \in \R^{D \times d},~P = [p_i]_{i=-\infty}^\infty \in \R^{D \times \infty}$.
		Here, $P$ is called a positional encoding.
		Since position-wise FNN and self-attention layers are permutation equivariant,
		a positional encoding is often added to break the equivariance when positional information is important.
		Sometimes, a learnable positional encoding is used, but we consider that $P$ is fixed since $P$ is infinite dimensional in our setting.
		Relative positional encoding~\citep{shaw_self-attention_2018} is another way to encode the positional information, but it requires extra trainable parameters.
		Therefore, we consider the absolute positional encoding in this paper.

		We define the class of transformers by
		\begin{align*}
			 & \mathcal{T}(M, U, D, H, L, W, S, B)                                                                          \\
			 &\quad := \left\{f_M \circ g_M \circ \dots \circ f_1 \circ g_1 \circ \enc_P \mid \norm{E}_\infty \leq B, \right.                    \\
			 &\quad \left. f_i \in \Psi(L, W, S, B), g_i \in \mathcal{A}(U_i, D, H, B)\right\},
		\end{align*}
		where FNN is applied column-wise.
            Thanks to the parameter sharing property, $F \in \mathcal{T}$ can represent a function from $\domain$ to $\R^\infty$ even though it has a finite number of parameters.
		In order to derive the estimation error for a model class $\mathcal{T}$,
		it is convenient to assume that there exists a constant $R > 0$ such that $\norm{f}_\infty \leq R$ for any $f \in \mathcal{F}$
		since this assumption ensure the sub-Gaussianity of $f(X^{(i)})$.
		To ensure this property, we define the class of (clipped) Transformer networks by
		\begin{align}
			\mathcal{T}_R := \qty{\tilde F = \clip_R \circ F \mid F \in \mathcal{T}},\label{eq:transformer-networks}
		\end{align}
		where $\clip_R(x) := R \wedge (x \vee -R)$ is applied element-wise. Note that $\clip_R$ can be realized by ReLU units.

        For simplicity, we consider a modified version of the original architecture in~\citet{vaswani_attention_2017}.
		That is, we consider the multilayer FNNs without skip connection instead of the single layer FNNs with skip connection.
		However, our argument can be applied to the original architecture with a slight modification.
		See Appendix~\ref{sec:modified} for details.

\section{Function Spaces}
	In this paper, we assume the true function $F^\circ$ is shift-equivariant.
	A function $F: \Omega \to \R^{d' \times \infty}$ is called shift-equivariant if $F$ satisfies
	\begin{align*}
		F(\Sigma_j(X)) = \Sigma_j(F(X)),
	\end{align*}
	for any $j \in \Z$ and $X \in \Omega$.
	Such equivariance appears in various applications such as natural language processing, audio processing, and time-series analysis.
	We also assume that $F_i(X) := (F(X))_i$ is included in a certain function class which is defined in this section.

	\subsection{Anisotropic and Mixed Smoothness}
		First, we introduce the $\gamma$-smooth function class.
		This is an extension of the function class in~\citet{okumoto_learnability_2022}
		to the situation where the inputs are bidirectional sequences of tokens.
		For $r \in \Z_0^{d\times \infty}$, we define $\psi_{r_{ij}}: [0, 1] \to \R$ by
		\begin{align*}
			\psi_{r_{ij}}(x) & := \begin{cases}
				\sqrt{2}\cos(2\pi\abs{r_{ij}}x) & \quad (r_{ij} < 0), \\
				1                               & \quad (r_{ij} = 0), \\
				\sqrt{2}\sin(2\pi\abs{r_{ij}}x) & \quad (r_{ij} > 0),
			\end{cases}
		\end{align*}
		and $\psi_r:\domain \to \R$ by $\psi_r(X) = \prod_{i=1}\prod_{j=1} \psi_{r_{ij}}(X_{ij})$.
		Since $\qty{\psi_r}_{r \in \Z_0^{d\times \infty}}$ is a complete orthonormal system of $L^2(\domain)$, any $f\in L^2(\domain)$ can be expanded as
		$f = \sum_{r\in \Z_0^{d\times \infty}} \ev{f, \psi_r}\psi_r$.
		For $s \in \N_0^{d\times \infty}$, define $\delta_s(f)$ as
		\begin{align*}
			\delta_s(f) = \sum_{r\in \Z_0^{d\times \infty}, \floor{2^{{s_{ij}-1}}}\leq r_{ij}<2^{s_{ij}}}\ev{f, \psi_r}\psi_r.
		\end{align*}
		This quantity represents the frequency component of $f$ with frequency $\abs{r_{ij}} \sim 2^{s_{ij}}$ for each coordinate.
		Then, we define the $\gamma$-smooth function class as follows.
		\begin{definition}[$\gamma$-smooth function class]
			For a given $\gamma:\N_0^{d\times \infty}\to \R$ which is monotonically non-decreasing with respect to each coordinate and $p\geq 2, \theta \geq 1$,
			we define the $\gamma$-smooth function space as follows:
			\begin{align*}
				\mathcal{F}_{p, \theta}^\gamma(\domain) & := \qty{f \in L^2(\domain) \mid \norm{f}_{\mathcal{F}_{p, \theta}^\gamma} < \infty},
			\end{align*}
			where the norm $\norm{f}_{\mathcal{F}_{p, \theta}^\gamma}$ is defined as
			\begin{align*}
				\norm{f}_{\mathcal{F}_{p, \theta}^\gamma} & := \qty(\sum_{s\in \N_0^{d\times \infty}} 2^{\theta \gamma(s)} \norm{\delta_s(f)}_{p, P_X}^\theta)^{1/\theta}.
			\end{align*}
			We also define the finite dimensional version of $\gamma$-smooth function space $\mathcal{F}_{p, \theta}^\gamma([0, 1]^{d\times l})$ for $l \in \N$ in the same way.
		\end{definition}
		Since $\delta_s(f)$ represents the frequency component of $f$ with frequency $\abs{r_{ij}} \sim 2^{s_{ij}}$ and weight $2^{\gamma(s)}$ is imposed on $\norm{\delta_s(f)}_p$,
		$\gamma$ controls the amplitude of each frequency component.

		As a special case of $\gamma$, we consider the mixed and anisotropic smoothness.
		\begin{definition}[Mixed and anisotropic smoothness]
			For $a \in \R_{>0}^{d\times \infty}$, mixed smoothness and anisotropic smoothness is defined as follows:
			\vspace{-1\baselineskip}
			{
				\setlength{\leftmargini}{10pt}
				\begin{itemize}
					\setlength{\itemsep}{0mm}
					\item mixed smoothness:
						\begin{align*}
							\gamma(s) & = \ev{a, s}.
						\end{align*}
					\item anisotropic smoothness:
						\begin{align*}
							\gamma(s) & = \max\qty{a_{ij}s_{ij}\mid i \in [d], j \in \Z}.
						\end{align*}
				\end{itemize}
			}
		\end{definition}
		The parameter $a$ represents the smoothness for the coordinate $X_{i, j}$.
		That is, if $a_{ij}$ is large, the function is smooth with respect to the variable $X_{i, j}$.
		In other words, small $a_{ij}$ implies that the function is not smooth towards the coordinate $(i, j)$ and $X_{ij}$ is an \textit{important feature}.

        When $d = 1$, $p=\theta=2$, and $P_X$ is the uniform distribution on $[0, 1]^l$, as shown in~\citet{okumoto_learnability_2022}, the anisotropic smooth function space $\mathcal{F}_{p, \theta}^\gamma([0, 1]^{l})$ includes the anisotropic Sobolev space:
        $
            \mathcal{W}_2^a := \qty{f \in L^2([0, 1]^l) \mid \sum_{i=1}^l \norm{\pdv[a_i]{f}{x_i}}_2^2 < \infty}
        $
        .
        Isotropic Sobolev spaces are a special case of anisotropic Sobolev spaces with $a_i = a_1~(\forall i\in [l])$.
        In that sense, $\mathcal{F}_{p, \theta}^\gamma([0, 1]^{d\times \infty})$ is an extension of the finite dimensional Sobolev space.
        
		Here, we define some quantities regarding the smoothness parameter $a$.
		Let $\bar a = \qty{\bar a_i}_{i=1}^\infty$ be the sorted sequence in the ascending order.
		That is, $\bar a = [a_{i_1, j_1}, \dots, a_{i_k, j_k}, \dots]$ satisfies $a_{i_k, j_k} \leq a_{i_{k+1}, j_{k+1}}$ for any $k \in \N$.
		Then, weak $l^\alpha$-norm for $\alpha > 0$ is defined by
		$
			\norm{a}_{wl^\alpha} := \sup_j j^\alpha \bar a_j^{-1},
		$
		and $\tilde a$ is defined by
		$
			\tilde a := \qty(\sum_{i=1}^\infty \bar a_i^{-1})^{-1}
		$.
  To simplify the notation, we define $a^\dagger = \bar a_1$ for the mixed smoothness and $a^\dagger = \tilde a$ for the anisotropic smoothness.

	\subsection{Piecewise Anisotropic and Mixed Smoothness}
		The mixed and anisotropic smooth functions represent the situations where the smoothness depends on the direction.
		However, the smoothness does not depend on each input. That is, the position of important tokens is fixed for \textit{any input}.
		This is not the case in practical situations.
		For example, in natural language processing, the positions of important words should change if a meaningless word is inserted in the input sequence.
		Therefore, it is natural to assume that the smoothness for each coordinate changes depending on each input.
		To consider such situations, we define a novel function class called \textit{piecewise} $\gamma$-smooth function class.
		\begin{definition}[Piecewise $\gamma$-smooth function class]
			For an index set $\Lambda$, let $\qty{\Omega_\lambda}_{\lambda\in \Lambda}$ be a disjoint partition of $\Omega$. That is, $\qty{\Omega_\lambda}_{\lambda\in \Lambda}$ satisfies
			\begin{align*}
				\Omega = \bigcup_{\lambda \in \Lambda} \Omega_{\lambda},~\Omega_\lambda \cap \Omega_{\lambda'} = \emptyset\quad (\lambda\neq \lambda').
			\end{align*}
			For $V \in \N$ and a set of bijections $\qty{\pi_\lambda}_{\lambda \in \Lambda}$ between $[2V+1]$ and $[-V:V]$,
			define $\Pi_\lambda: \R^{d\times [-V:V]}\to \R^{d\times (2V+1)} $ and $\Pi: \Omega \to \R^{d \times (2V+1)}$ by
			\begin{align*}
				\Pi_\lambda([x_{-V}, \dots, x_{V}]) & := [x_{\pi_\lambda(1)}, \dots, x_{\pi_\lambda(2V+1)}],      \\
				\Pi(X)                              & := \Pi_\lambda(X[-V:V]) \text{ if } X \in \Omega_\lambda.
			\end{align*}
			Then, for $p \geq 2, \theta \geq 1$ and $\gamma:\N_0^{d\times \infty} \to \R$, the function class with piecewise $\gamma$-smoothness is defined as follows:
			\begin{align*}
				\mathcal{P}_{p, \theta}^\gamma(\Omega) & := \left\{g = f \circ \Pi \mid \right.                                                                                             \\
				                                       & \left. f \in \mathcal{F}_{p, \theta}^\gamma([0, 1]^{d\times (2V+1)}), \norm{g}_{\mathcal{P}_{p, \theta}^\gamma} < \infty \right\},
			\end{align*}
			where the norm $\norm{g}_{\mathcal{P}_{p, \theta}^\gamma}$ is defined by
			\begin{align*}
				\norm{g}_{\mathcal{P}_{p, \theta}^\gamma} & := \qty(\sum_{s \in \N_0^{d\times [-V:V]}} 2^{\theta\gamma(s)}\norm{\delta_s(f) \circ \Pi}_{p, P_X}^\theta)^{1/\theta}.
			\end{align*}
		\end{definition}
		On each domain $\Omega_\lambda$, a piecewise $\gamma$-smooth function $g$ can be seen as the restriction of a certain $\gamma$-smooth function $g_\lambda$.
		In addition, considering the mixed or anisotropic smoothness, the smoothness parameter of $g_\lambda$ is a permutation of the original smoothness parameter $a$.
		Therefore, the relatively smooth directions of $g$ change depending on each input. This situation is shown in Fig.~\ref{fig:piecewise}.

		In this paper, we assume that there exists an \textit{importance function}, defined as follows.
        \begin{definition}[importance function]
            A function $\mu: \Omega \to \R^{\infty}$ is called an \textit{importance function} for $\qty{\Omega_\lambda}_{\lambda \in \Lambda}$ if $\mu$ satisfies
            \begin{align*}
                \Omega_\lambda & =\{X \in \Omega \mid \mu(X)_{\pi_{\lambda}(1)} > \cdots > \mu(X)_{\pi_{\lambda}(2V+1)}\}.
            \end{align*}
	\end{definition}
        Here, we briefly explain the intuition behind the definition. 
        For $X\in \Omega$, let $X' = \Pi(X)$. Assume that $x'_i$ is more important than $x'_{i+1}$.
        From the definition of $\Pi$, we have $x'_{i} = x_{\pi_{\lambda}(i)}$ when $X \in \Omega_\lambda$.
        Therefore, the token $x_{\pi_{\lambda}(i)}$ is more important than $x_{\pi_{\lambda}(i+1)}$.
        The definition of the importance function reflects this relationship.
	We also assume that an importance function $\mu$ is \textit{well-separated}. That is, $\mu$ satisfies
		\begin{align}
			\mu(X)_{\pi_{\lambda}(i)} \geq \mu(X)_{\pi_{\lambda}(i+1)} + c i^{-\beta}, \label{eq:importance}
		\end{align}
		for any $X \in \Omega_\lambda$, where $c, \beta > 0$ is a constant. This implies that the probability that $X$ satisfies $\mu(X)_i \simeq \mu(X)_j~(i\neq j)$ is zero.
            Similar assumption can be found in the analysis of infinite dimensional PCA~\citep{hall_methodology_2007}.
            We will assume later that $\mu$ has mixed or anisotropic smoothness.

\section{Approximation Error Analysis}\label{sec:approximation}
	In this section, we study the approximation ability of Transformers in the case that the target function has (piecewise) anisotropic or mixed smoothness.
	Our analysis shows that Transformer networks can approximate shift-equivariant functions under appropriate assumptions even if inputs and outputs are infinite dimensional.
	This is in contrast to the analysis for any continuous functions~\citep{yun_are_2020}, where the number of parameters increases exponentially with respect to the input dimensionality.

        \subsection{Mixed and Anisotropic Smoothness}
		First, we derive the approximation error of Transformer networks for the mixed and anisotropic smoothness.
		In our analysis, we assume the following. Similar assumption can be found in~\citet{okumoto_learnability_2022}.
		\begin{assumption}\label{assumption:anisotropic}
			The true function $F^\circ$ is shift-equivariant and satisfies
			\begin{align*}
				F^\circ_0 & \in U(\mathcal{F}_{p, \theta}^\gamma),~\norm{F_0}_{\infty} \leq R,
			\end{align*}
			where $R > 0$ is a constant and $\gamma$ is mixed or anisotropic smoothness.
			In addition, the smoothness parameter $a$ satisfies $\norm{a}_{wl^\alpha} \leq 1$ for some $0 < \alpha < \infty$
			and $a_{ij} = \Omega(\log(\abs{j}+1))$.
			For the mixed smoothness, we also assume $\bar a_1 < \bar a_2$.
		\end{assumption}
		Note that the assumption implies that $F^\circ_i~(i\neq 0)$ also have mixed or anisotropic smoothness due to the shift-equivariance.
		The weak $l^\alpha$ norm condition implies the sparsity.
		Since the smoothness should increase in polynomial order if $\norm{a}_{wl^\alpha} \leq 1$,
		most of $a_{ij}$s should be large, which means there exist few important features in an input.
		This assumption is partially supported by the fact that attention weights are often sparse in practice~\citep{likhosherstov_expressive_2021}.
		On the other hand, the condition $a_{ij} = \Omega(\log(\abs{j} + 1))$ implies the locality.
		That is, $a_{ij}$ should be large if $\abs{j} \gg 1$, and thus $x_j$ is not important.
		This is introduced due to the importance of local context in some applications such as natural language processing.

		Under this assumption, the approximation error of Transformer networks is evaluated as follows.
		\begin{theorem}\label{thm:approximation}
			Suppose that the target function $F^\circ$ satisfies Assumption~\ref{assumption:anisotropic}.
			Then, for any $T > 0$, there exists a transformer network $\hat F \in \mathcal{T}(M, U, D, H, L, W, S, B)$ such that
			\begin{align*}
				\norm{\hat F_i - F^\circ_i}_{2, P_X} & \lesssim 2^{-T}
			\end{align*}
			for any $i \in \Z$, where $\phi = \frac{1}{2 U_1 + 1}$, and
			\begin{align}
				\begin{split}
					M      & = 1,~\log U_1 \sim T,~D \sim T^{1/\alpha},~H \sim T^{1/\alpha},       \\
					L      & \sim \max\qty{T^{2/\alpha}, T^2},~W \sim T^{1/\alpha}2^{T / a^\dagger}, \\
					\log B & \sim \max\qty{T^{1/\alpha}, T},                                       \\
					S      & \sim T^{2/\alpha}\max\qty{T^{2/\alpha}, T^2}2^{T / a^\dagger}            \\
					p_i    & =\qty[0, \dots, 0, \cos(i\phi), \sin(i\phi)]^\top.
				\end{split}\label{eq:hyperparameters}
			\end{align}
		\end{theorem}

		The proof can be found in Appendix~\ref{sec:proof-approximation}.
		The results show that even though inputs and outputs are infinite dimensional, the approximation error can be bounded by $N^{-a^{\dagger}}$ ignoring poly-log factor, where $N$ denotes the number of parameters,
		since the total number of parameters is bounded by $N \lesssim M(S+HD^2) \sim 2^{T/a^\dagger}$ ignoring poly-log factor.
		This is in contrast to FNNs, where the number of parameters should increase at least linearly with the input and output length.
		The \textit{parameter sharing} property and \textit{feature extraction} ability of Transformers play an essential role in the proof.
		For anisotropic smoothnesss, the result can be seen as an extention of the result of FNN for anisotropic Besov space~\cite{suzuki_deep_2021} to infinite dimensional input and sequence-to-sequence setting.

		In addition, positional encoding is an important factor
		in extracting local context with limited interaction among tokens compared to FNN and CNN.
		Due to the shift-equivariance, Transformer networks should extract important tokens for each output by relative position.
		Since we use absolute positional encoding, it is not trivial to show that Transformers have such capability.
		Indeed, existing works~\citep{edelman_inductive_2022,gurevych_rate_2022} used absolute position to focus tokens and their analysis cannot be applied to multiple output setting with shift-equivariance.
		To overcome this issue, we adopt the \textit{sinusoidal positional encoding}
		since the shift operation $p_i \to p_{i + j}$ can be represented by linear transformation, as mentioned in~\citet{vaswani_attention_2017}.
		This allows the self-attention mechanism to attend by relative position as shown in Fig.~\ref{fig:approximation}.
		In addition, the size of the positional encoding in~\citep{edelman_inductive_2022,gurevych_rate_2022} depends on the size of the receptive field. For example, \citet{gurevych_rate_2022} used the standard basis as positional encoding and the size of the positoinal encoding grows linearly with respect to the input length.
		On the other hand, we show that fixed length positional encoding is enough for feature extraction by adjusting the scale appropriately inside the self-attention mechanism.

		\begin{figure}
			\centering
			\includegraphics[width=0.9\columnwidth]{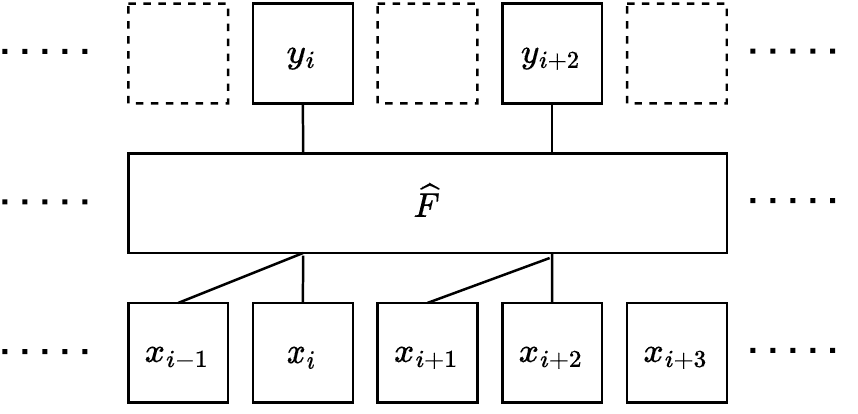}
			\caption{The self-attention mechanism can attend by relative position.
				In this diagram, each token attend to the previous token and itself.}
			\label{fig:approximation}
		\end{figure}

		\begin{remark}
			The results in Theorem~\ref{thm:approximation} can be extended to the 2D input setting like image processing
			by modifying the positional encoding as $p_{ij} = [0, \dots, 0, \cos(i\phi), \sin(i\phi), \cos(j\phi), \sin(j\phi)]$,
			where $i, j$ represent the row index and column index of the token $x_{ij}$, respectively.
			That is, Transformer can approximate both vertically and horizontally shift-equivariant functions with 2D inputs under appropriate assumptions.
			Similarly, the other results in this paper can be extended to the 2D input setting.
			Therefore, to some extent, our analysis explains the practical success of Transformers in the field of image processing~\citep{dosovitskiy_image_2021}.
		\end{remark}
	\subsection{Piecewise Smoothness}
		Next,  we derive the approximation error for the piecewise mixed and anisotropic smoothness, where the smoothness depends on each input.
            For the piecewise smoothness, we assume the following.
		\begin{assumption}\label{assumption:piecewise}
			The true function $F^\circ$ is shift-equivariant and satisfies
			\begin{align*}
				F^\circ_0 & \in U(\mathcal{P}_{p, \theta}^\gamma),~\norm{F^\circ_0}_\infty \leq R,
			\end{align*}
			where $R$ is a constant, $\gamma$ is  mixed or anisotropic smoothness, and the smoothness parameters $a$ satisfies $a_{ij} = \Omega(j^{\alpha})$ and $\norm{a}_{wl^\alpha} \leq 1$ for some $0 < \alpha < \infty$.
			For the mixed smoothness, we also assume $\bar a_1 < \bar a_2$.
			In addition, we assume the importance function $\mu$ satisfies Assumption~\ref{assumption:anisotropic} with $p = \infty, R=1$.
		\end{assumption}

		For piecewise $\gamma$-smooth functions, it is necessary to extract important features depending on each input as shown in Fig.~\ref{fig:piecewise}.
		Since $\Pi$ is not a linear operator,
		linear dimension reduction methods such as PCA and linear convolutional layers cannot approximate $\Pi$.
		However, due to the dynamical feature extraction ability of self-attention mechanism, Transformers can approximate $\Pi$ and
		achieve similar approximation error results as in the fixed smoothness setting.
		\begin{theorem}\label{thm:approximation-piecewise}
			For $d' = O(T^{2(\beta + 1)/\alpha}\log V)$, let $\{u_i\}_{i=1}^{2V + 1} \subset \R^{d'}$ be approximately orthonormal vectors which satisfies 
                \begin{align*}
        		\abs{\ev{u_i, u_j}} & \leq \varepsilon,~\ev{u_i, u_i} = 1.
        	\end{align*}
			In addition, We define $u_i$ for $i \notin [2V + 1]$ by $u_i := u_{i \mod (2V + 1)}$.
			Suppose that the target function $F^\circ$ satisfies Assumption~\ref{assumption:piecewise}.

                Then, for any $T > 0$, there exists a transformer $\hat F \in \mathcal{T}(M, U, D, H, L, W, S, B)$ such that
			\begin{align*}
				\norm{\hat F_i - F^\circ_i}_{2, P_X} \lesssim 2^{-T}
			\end{align*}
			for any $i \in \Z$, where
			\begin{align}
				\begin{split}
					M      & \sim T^{1/\alpha},~D \sim T^{2(\beta + 1)/\alpha}\log V,\\
					\log U_1 &\sim \log T,~U_i = V~(i\geq 2),~H \sim (\log T)^{1/\alpha},\\
					L &\sim \max\qty{T^{2/\alpha}, T^2},~W \sim T^{1/\alpha}2^{T/a^\dagger},\\
					\log B &\sim \max\qty{T^{1/\alpha}, T, \log \log V}\\
					S &\sim T^{2/\alpha}\max\qty{T^{2/\alpha}, T^2}2^{T / a^\dagger},\\
                    p_i &= [0, \dots, 0, 1, \cos(i\phi), \sin(i\phi), u_i^\top]^\top.
				\end{split}\label{eq:hyperparameters-piecewise}
			\end{align}
		\end{theorem}

		See Appendix~\ref{sec:proof-approximation-piecewise} for the proof.
            This theorem implies that multilayer Transformer networks can approximate mixed and anisotropic smooth functions even if inputs and outputs are infinite dimensional, and the smoothness structure depends on each input.
		In addition, the convergence rate is $N^{-a^\dagger}$, which is the same as in Theorem~\ref{thm:approximation}. This can be realized by the \textit{dynamical} feature extraction ability of the self-attention mechanism.
		See Fig.~\ref{fig:piecewise} for an illustration of the feature extraction mechanism.
		The construction consists of two phases.
		First, the first layer of the Transformer network approximates the importance function $\mu$.
		Next, the Transformer network selects important tokens by the self-attention mechanism based on the estimated importance.
		Thanks to the softmax operation in the self-attention mechanism, Transformers can extract the \textit{most} important token.
		However, it is difficult to extract the second (and subsequent) most important tokens.
		To overcome this issue, we developed a novel approximately orthonormal basis coding technique for the positional encoding to memorize already extracted tokens.
		Approximately orthonormal vectors can be constructed via random sampling, as shown in Lemma~\ref{lem:approximately}.

		\begin{figure}
			\centering
			\includegraphics[width=\columnwidth]{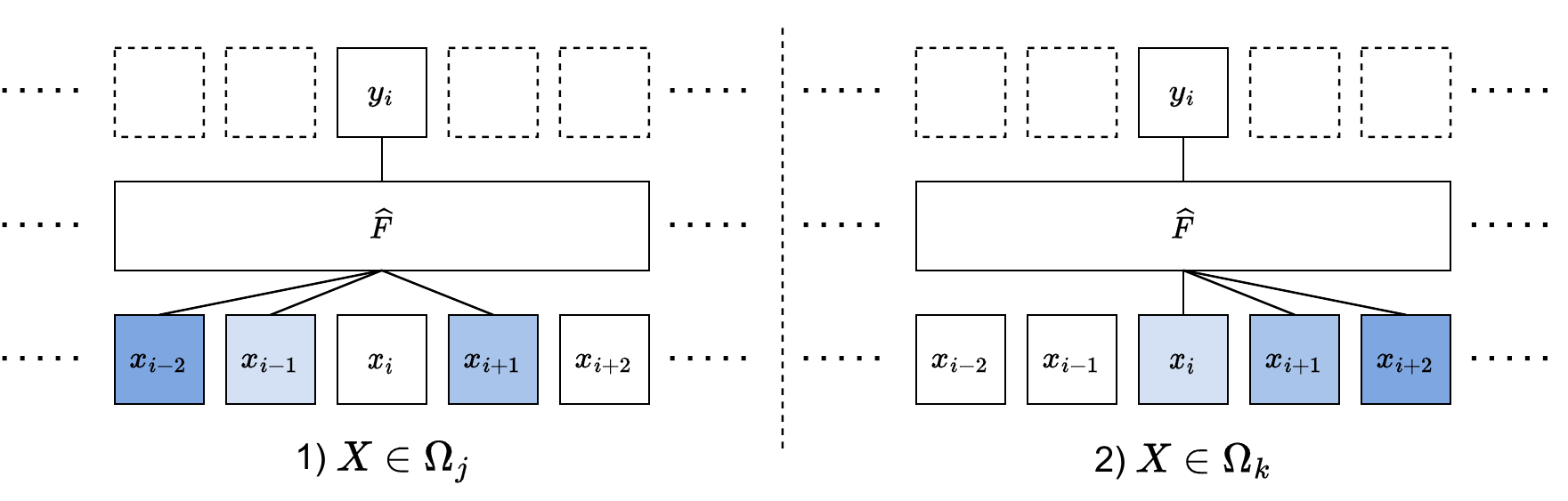}
                \vspace*{-0.5cm}
			\caption{For piecewise $\gamma$-smoothness, the position of important tokens depends on each input.
				We show important tokens in darker color.
				In the case of $X \in \Omega_j$, the most important token to $y_i$ is $x_{i-2}$
				and in the case of $X \in \Omega_k$, $x_{i+2}$ is the most important.
				The self-attention mechanism can switch its attention (represented by black lines) depending on the importance of tokens.}
			\label{fig:piecewise}
		\end{figure}
  
		In this setting, the attention weights \textit{dynamically} change depending on each input.
		This is in contrast to the existing works~\citep{edelman_inductive_2022, okumoto_learnability_2022}, which studied the feature extraction ability of Transformer networks and CNNs, respectively.
		Our result matches the empirical findings~\citep{likhosherstov_expressive_2021} and
		Theorem~\ref{thm:approximation-piecewise} theoretically supports the \textit{dynamical} feature extraction ability of the self-attention mechanism by considering the novel function class.

\section{Estimation Error Analysis}\label{sec:estimation}
	In this section, we show that Transformer networks can achieve polynomial estimation error rate and avoid the curse of dimensionality.

	To evaluate the variance of estimators, the covering number is often used to capture the complexity of model classes.
	\begin{definition}[Covering Number]
		For a normed space $\mathcal{F}$ with a norm $\norm{\cdot}$, the $\delta$-covering number is defined as
		\begin{align*}
			\mathcal{N}(\mathcal{F}, \delta, \norm{\cdot}) & := \inf\left\{n \in \N\mid \exists (f_1, \dots, f_n) \in \mathcal{F},\right.              \\
			                                               & \left. \forall f \in \mathcal{F}, \exists i \in [n], \norm{f_i - f} \leq \delta \right\}.
		\end{align*}
	\end{definition}

	For non-parametric regression problems with infite dimensional inputs and outputs, the estimation error of an ERM estimator is evaluated as follows.
	\begin{theorem}\label{thm:excess-risk}
		For a given class $\mathcal{F}$ of functions from $\domain$ to $\R^\infty$, let $\hat F \in \mathcal{F}$ be an ERM estimator which minimizes the empirical cost.
		Suppose that there exists a constant $R > 0$ such that $\norm{F^\circ}_\infty \leq R$, $\norm{F}_\infty \leq R$ for any $F \in \mathcal{F}$, and $\mathcal{N}(\mathcal{F}, \delta, \norm{\cdot}_\infty) \geq 3$.
		Then, for any $0 < \delta < 1$, it holds that
		\begin{align*}
			 & R_{l, r}(\hat F, F^\circ) \leq 4\inf_{F\in \mathcal{F}} \frac{1}{r - l + 1}\sum_{i=l}^r \norm{F_i - F^\circ_i}_{2, P_X}^2 \\
			 & \quad + C((R^2+\sigma^2) \frac{\log \mathcal{N}(\mathcal{F}, \delta, \norm{\cdot}_\infty)}{n} + (R + \sigma)\delta),
		\end{align*}
		where $C > 0$ is a global constant.
	\end{theorem}
	This theorem is a direct extension of Lemma 4 in~\citet{schmidt-hieber_nonparametric_2020} and Theorem 2.6 in~\citet{hayakawa_minimax_2020} to the multiple output setting.
	The proof can be found in Appendix~\ref{sec:peoof-excess-risk}.

	By carefully evaluating the complexity of the class of Transformer networks, we have the following bound on the log covering number of Transformer networks.
	\begin{theorem}\label{thm:covering-number}
		For given hyperparameters $M, U, D, H, L, W, S, B$, assume that $B \geq 1$ and $\norm{P}_\infty \leq B$. Then, we have the following log covering number bound:
		\begin{align*}
			 & \log \mathcal{N}(\mathcal{T}(M, U, D, H, L, W, S, B), \delta, \norm{\cdot}_\infty) \\
			 & \quad \lesssim  M^3L(S + HD^2)\log\qty(\frac{DHLWB}{\delta}).
		\end{align*}
	\end{theorem}
        See Appendix~\ref{sec:proof-covering-number} for the proof.
	Interestingly, the covering number bound does not depend on the dimensionality of inputs and outpus, and the width of the sliding window.
	This is because the number of parameters does not depend on these quantities due to the parameter sharing property and
	the magnitude of the hidden states are independent of the window size since the attention weights $A$ are normalized as $\norm{A}_1 = 1$.
	The parameter sharing property, on the other hand, leads to the limited interaction among tokens.
	This makes approximation analysis difficult, and thus it is necessary to design positional encoding carefully as mentioned in Section~\ref{sec:approximation}.

	Combining above results, we have the following estimation error bound for the mixed and anisotropic smoothness.
	\begin{theorem}\label{thm:estimation-error}
		Suppose that Assumption~\ref{assumption:anisotropic} holds.
		Let $\hat F$ be an ERM estimator in $\mathcal{T}_R(M, U, D, H, L, W, S, B)$, where $M, U, D, H, L, W, S, B$ is defined as~\eqref{eq:hyperparameters} and $T = \frac{a^\dagger}{2a^\dagger + 1} \log n$.
		Then, for any $l, r \in \Z$, we have
		\begin{align*}
			R_{l, r}(\hat F, F) & \lesssim n^{-\frac{2a^\dagger}{2 a^\dagger + 1}} (\log n)^{2/\alpha + 2 + \max\qty{4/\alpha, 4}}.
		\end{align*}
	\end{theorem}
	This can be shown by letting $\delta = 1/n$ in Theorem~\ref{thm:excess-risk}.
        See Appendix~\ref{sec:proof-estimation} for the proof.
	The results show that the convergence rate of the estimation error does not depend on the input dimensionality and the output size if the smoothness of the target function has sparse structure.
	This implies that Transformers can avoid the curse of dimensionality.
	When $d = 1$, this convergence rate matches that for CNNs~\citep{okumoto_learnability_2022} for single output setting.
	For anisotropic smoothness, this rate also matches, up to poly-log order, that of FNNs in the finite dimensional setting, which is known to be minimax optimal~\citep{suzuki_deep_2021}.
	That is, Transformers can achieve near-optimal rate in a minimax sense.

	In addition, the Transformer architecture including the positional encoding $P$ does not depend directly on the smoothness structure $a$.
	This implies that Transformer networks can find important features and select them \textit{adaptively} to the smoothness of the target function by learning the intrinsic structure of the target function.

	For the piecewise smoothness, we have the following estimation error bound.
	\begin{theorem}\label{thm:estimation-piecewise}
		Suppose that Assumption~\ref{assumption:piecewise} holds.
		Let $\hat F$ be an ERM estimator in $\mathcal{T}_R(M, U, D, H, L, W, S, B)$, where $M, U, D, H, L, W, S, B$ is defined as~\eqref{eq:hyperparameters-piecewise} and $T = \frac{a^\dagger}{2a^\dagger + 1} \log n$.
		Then, for any $l, r \in \Z$, we have
		\begin{align*}
			R_{l, r}(\hat F, F) & \lesssim n^{-\frac{2a^\dagger}{2 a^\dagger + 1}} (\log n)^{5/\alpha + 2 + \max\qty{4/\alpha, 4}}(\log V)^3.
		\end{align*}
	\end{theorem}
	See Appendix~\ref{sec:proof-estimation-piecewise} for the proof.
	This convergence rate is the same as in Theorem~\ref{thm:estimation-error} up to poly-log order if $V = \poly(n)$.
	This means that Transformers can avoid the curse of dimensionality even if the smoothness architecture depends on each input.

	In addition, the Transformer architecture does not depend on the partition $\qty{\Omega_\lambda}_{\lambda \in \Lambda}$ and the importance function $\mu$.
	This means that Transformers can adapt to the intrinsic structure of the target function and realize the \textit{dynamical} feature extraction according to the importance of tokens.
	This fact supports the practical success of Transformer networks in a wide range of applications with various structures.    

    \section{Numerical Experiments}
    \begin{figure}
        \centering
            \centering
            \includegraphics[width=0.9\columnwidth]{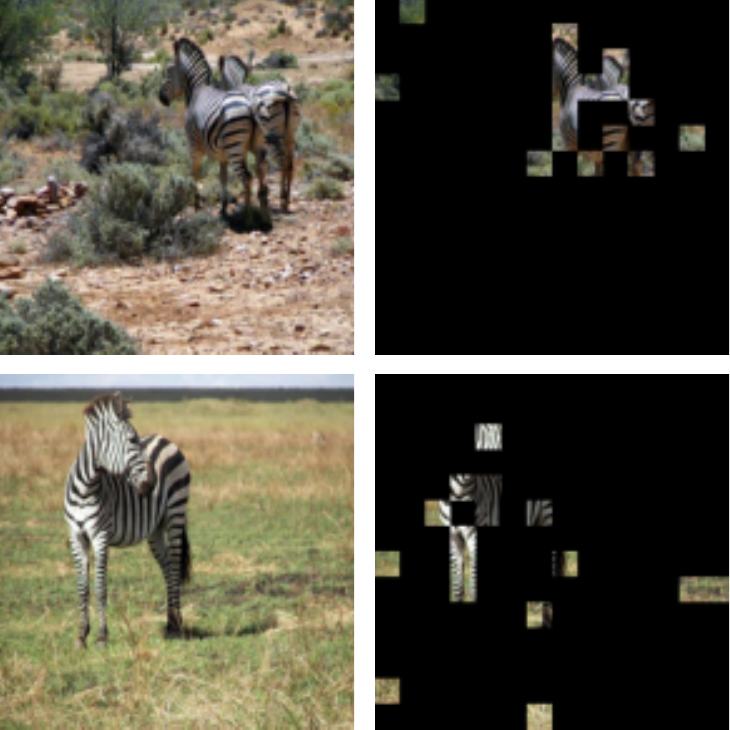}
            \caption{Two zebra images (left) and the corresponding images with 180 / 196 patches masked (right). }
        \label{fig:images}
    \end{figure}

    The assumptions in this paper essentially impose the sparsity of important features. To verify this, we conducted some numerical experiments using masked images as inputs.
    In this experiment, we prepared a pre-trained model (ViT-Base model~\cite{dosovitskiy_image_2021}) and two images of zebras in Fig.~\ref{fig:images} from the validation set of ImageNet-1k~\cite{imagenet15russakovsky}. We divided each image into $14 \times 14$ tokens and masked each token in turn. At each step, a token to be masked is selected using a greedy algorithm to maximize the predicted probability of the correct class by the pre-trained model. Since masking informative features strongly affects the predicted probability, important features will remain unmasked near the end of the procedure.

    As shown in Fig.~\ref{fig:loss_hist}, the predicted probability remains high and the model can classify the image correctly even if about 90\% of the input is masked. This means that a small fragment of the image is important for prediction.
    In addition, Fig.~\ref{fig:images} demonstrates that the patterns of unmasked, i.e., important features in the two images differ significantly.
    This implies important features change depending on each input and the piecewise smoothness describes more practical situations than the fixed smoothness.

    \begin{figure}
            \centering
            \includegraphics[width=0.9\columnwidth]{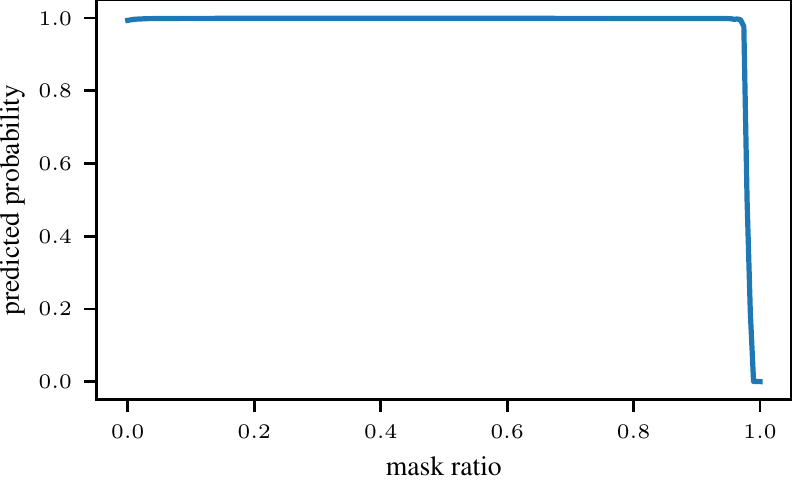}
            \caption{The predicted probability of the correct class for the top left image in Fig.~\ref{fig:images}. The predicted probability remains high even when most of the images are masked.}
            \label{fig:loss_hist}
    \end{figure}
    
        \section{Conclusion}
	In this study, we have investigated the learnability of Transformer networks for sequence-to-sequence functions with infinite dimensional inputs.
	We have shown that Transformer networks can achieve a polynomial order convergence rate of estimation error when the smoothness of the target function has sparse structure.
	In addition, we have considered the situations where the smoothness depends on each input.
	Then, we have shown that Transformer networks can avoid the curse of dimensionality by switching the focus of attention based on input.
	We believe that our theoretical analysis provides a new insight into the nature of Transformer architecture.



	\section*{Acknowledgements}
        ST was partially supported by Fujitsu Ltd. TS was
        partially supported by JSPS KAKENHI (20H00576) and
        JST CREST.




	\bibliography{main}
	\bibliographystyle{icml2023}

	\newpage
	\appendix
	\onecolumn

	\section{Notation List}\label{sec:notation-table}
\begin{table}[h]
	\caption{Notation list}
	\label{table:notation}
	\vskip 0.15in
	\begin{center}
		\begin{tabular}{|l|l|}
			\toprule
			Notation                                                                       & Definition                                                                           \\
			\midrule
			$n$                                                                            & sample size                                                                          \\
			$d$                                                                            & token dimension                                                                      \\
			$(X^{(i)}, y^{(i)})$                                                           & $i$-th observation                                                                   \\
			$\mathcal{D}^n = \qty{(X^{(i)}, y^{(i)})}_{i=1}^n$                             & training data                                                                        \\
			$P_X$                                                                          & data distribution                                                                    \\
			$\Omega$                                                                       & support of $P_X$                                                                     \\
			$\qty{\Omega_{\lambda}}_{\lambda \in \Lambda}$                                 & disjoint partition of $\Omega$                                                       \\
			$\mathcal{F}_{p, \theta}^\gamma$                                               & $\gamma$-smooth function class                                                       \\
			$\mathcal{G}_{p, \theta}^\gamma$                                               & piecewise $\gamma$-smooth function class                                             \\
			$\sigma$                                                                       & noise variance                                                                       \\
			$F^\circ$                                                                      & true function                                                                        \\
			$a$                                                                            & smoothness parameter                                                                 \\
			$\bar a = [a_{i_1, j_1}, \dots, a_{i_k, j_k}, \dots]$                     & smoothness parameter sorted in the ascending order                                   \\
			$\tilde a$                                                                     & $\qty(\sum_{i=1}^\infty \bar a_i^{-1})^{-1}$                                         \\
			$a^\dagger$                                                                    & $\bar a_1$ for the mixed smoothness and $\tilde a$ for the anisotropic smoothness    \\
			$\mu$                                                                          & importance function                                                                  \\
			$\eta$                                                                         & ReLU activation function                                                             \\
			$I(T, \gamma)$                                                                 & $\qty{(i, j)\mid \exists s \in \N_0^{d\times \infty}, s_{ij} \neq 0, \gamma(s) < T}$ \\
			$I_j(T, \gamma)$                                                               & $\qty{i \mid (i, j) \in I(T, \gamma)}$                                               \\
			$e_i \in \R^{[l:r]}~(i \in [l:r])$                                           & standard basis of $\R^{[l:r]}$                                                      \\
			$\delta_{i, j} \in \R^{[l_1: r_1] \times [l_2: r_2]}~(i \in [l_1, r_1],~j \in [l_2: r_2])$ & $\delta_{i, j}:=e_i e_j^\top$                                                        \\
			$x \close{\epsilon} y~(x,y\in \R^d)$                                           & $\norm{x-y}_2 \lesssim \epsilon$                                                     \\
			$1_A$                                           & 1 if $A$ is true and 0 if $A$ is false                                                     \\
			\bottomrule
		\end{tabular}
	\end{center}
\end{table}

\section{Extension to the Original Architecture}\label{sec:modified}
In this paper, we consider the multilayer FNN without skip connection while the FNN in the original architecture~\cite{vaswani_attention_2017} uses single hidden layer with skip connection.
However, our argument can be applied to the original structure with slight modifications as follows.
\begin{itemize}
	\item (Skip connection)
	      Since there exists an FNN $f:\R^d \to \R^d$ with one hidden layer of width $2d$
	      which works as an identity map, we can cancel out the skip connection using $f$.
	      This modification does not change the order of hyperparameters, and we can obtain the same convergence rate as the original one.
	\item (One hidden layer)
	      By letting $V_h = O$ in a self-attention layer, the self-attention layer behaves as an identity map due to the skip connection in the self-attention layer. Therefore,
	      $L$-layer Transformer blocks (one block consists of an FNN layer with one hidden layer and a self-attention layer) can represent an
	      $L$-layer FNN. Then, we can obtain the same convergence rate as the original one up to poly-log order.
\end{itemize}

\section{Auxiliary Lemmas}
\begin{lemma}\label{lem:softmax}
	For $\theta \in \R^d$, assume that there exist an index $i^* \in [d]$ and $\delta > 0$ such that $\theta_{i^*} \geq \theta_{i} + \delta$ for any $i\neq i^*$.
	Then, we have
	\begin{align*}
		\norm{\softmax(\theta)-e_{i^*}}_1 & \leq 2de^{-\delta}.
	\end{align*}
\end{lemma}
\begin{proof}
	For $i\neq i^*$, the assumption $\theta_{i^*} \geq \theta_{i} + \delta$ yields that
	\begin{align*}
		0 & \leq \softmax(\theta)_i =\frac{e^{\theta_i}}{\sum_{j=1}^d e^{\theta_j}} \leq e^{\theta_i - \theta_{i^*}} \leq e^{-\delta}.
	\end{align*}
	For $i^*$, we have
	\begin{align*}
		0 & \leq 1 -  \softmax(\theta)_{i^*} = \sum_{i\neq i^*} \softmax(\theta)_{i} \leq (d-1)e^{-\delta}.
	\end{align*}
	Therefore,
	\begin{align*}
		\norm{\softmax(\theta)-e_{i^*}}_1 & = \abs{1 - \softmax(\theta)_{i^*}} + \sum_{i\neq i^*} \abs{\softmax(\theta)_i} \leq 2de^{-\delta},
	\end{align*}
	which completes the proof.
\end{proof}

\begin{lemma}\label{lem:softmax-lipschitz}
	For any $\theta, \theta' \in \R^d$, we have
	\begin{align*}
		\norm{\softmax(\theta) - \softmax(\theta')}_1 & \leq 2\norm{\theta - \theta'}_\infty.
	\end{align*}
\end{lemma}
\begin{proof}
	See Corollary A.7 in~\citet{edelman_inductive_2022}.
\end{proof}

\begin{lemma}\label{lem:approximately}
	For any $l \in \N$ and $\varepsilon > 0$, there exists $\varepsilon$-approximately orthonormal vectors $\qty{u_i}_{i=1}^l \subset \R^d~(d=O\qty(\frac{\log l}{\varepsilon^2}))$ which satisfies
	\begin{align}
		\begin{split}
			\abs{\ev{u_i, u_j}} & \leq \varepsilon,\\
			\ev{u_i, u_i}       & = 1,
		\end{split}\label{eq:approximately-orthonormal}
	\end{align}
	for any $i\neq j$.
\end{lemma}
\begin{proof}
	Let each component of $u_i \in \R^d$ independently follows Bernoulli distribution $\Probability((u_i)_j = \pm 1/\sqrt{d}) = 1/2$.
	Then, we have $\norm{u_i} = 1$. Since $\ev{u_i, u_j}$ can be seen as the sum of independent Bernoulli random variables, a Chernoff bound implies
	\begin{align*}
		\Probability(\abs{\ev{u_i, u_j}} > \varepsilon) & \leq 2\exp(-\frac{\varepsilon^2}{2}d).
	\end{align*}
	By letting $d = \frac{4\log 2l}{\varepsilon^2}$, we have
	\begin{align*}
		\Probability(\abs{\ev{u_i, u_j}} > \varepsilon) & \leq 2\exp(-2\log 2l) = \frac{1}{2l^2}.
	\end{align*}
	Since there exists at most $l^2/2$ pairs $(i, j)$, from the union bound,
	the probability that $\abs{\ev{u_i, u_j}} \leq \varepsilon$ for all $(i, j)$ such that $i\neq j$ is greater than $1-\frac{l^2}{2}\cdot \frac{1}{2l^2} = 3/4$.
	This means there exist approximately orthonormal vectors which satisfy Eq.~\eqref{eq:approximately-orthonormal}.
\end{proof}

\begin{lemma}\label{lem:lipschitz}
	For $B \geq 1$, let $f \in \Psi(L, W, B, S)$, and $g \in \mathcal{A}(U, D, H, B)$.
	Then, for any $r \geq 1$, $x, x' \in \R^{d}$ and $X, X' \in [-r, r]^{d \times \infty}$, we have
	\begin{align}
		\norm{f(x) - f(x')}_{\infty} & \leq (BW)^L \norm{x - x'}_{\infty} \leq (6HDBW)^{4L}r^2 \norm{x - x'}_{\infty},   \label{eq:fnn-lipschitz} \\
		\norm{g(X) - g(X')}_{\infty} & \leq 6HB^3D^4r^2 \norm{X - X'}_{\infty} \leq (6HDBW)^{4L}r^2 \norm{X - X'}_{\infty}. \notag
	\end{align}
\end{lemma}
\begin{proof}
	Since the ReLU activation function $\eta$ is $1$-Lipschitz continuous
	and $\norm{W_i x + b - (W_i x' + b)}_\infty \leq BW \norm{x - x'}_{\infty}$, $f$ is $(BW)^L$-Lipschitz continuous.
	This implies Eq.~\eqref{eq:fnn-lipschitz}.

	Let $\bar X := X[-U:U]$ and $A_h := \softmax((K_h\bar X)^\top (Q_hx_i))$.
	We define $\bar X', A'_h$ in the same way as $\bar X, A_h$.
	Then, we have
	\begin{align*}
		\norm{g(X)_i - g(X')_i}_\infty & \leq \norm{x_i - x'_i}_\infty + \sum_{h=1}^{H} \norm{V_h \bar X A_h - V\bar X' A'_h }_\infty                                               \\
		                               & \leq \norm{x_i - x'_i}_\infty + \sum_{h=1}^{H} \norm{V_h \bar X A_h - V\bar X A'_h }_\infty + \norm{V\bar X A'_h - V\bar X' A'_h }_\infty.
	\end{align*}
	From Lemma~\ref{lem:softmax-lipschitz}, we have
	\begin{align*}
		\norm{A_h - A_h'}_1 & \leq 2 \norm{(K_h \bar X)^\top(Q_h x_i) - (K\bar X')^\top(Q_h x'_i)}_\infty                                                                          \\
		                    & \leq 2 \norm{(K_h \bar X)^\top(Q_h x_i) - (K\bar X)^\top(Q_h x'_i)}_\infty + 2 \norm{(K_h \bar X)^\top(Q_hx'_i) - (K_h\bar X')^\top(Q_hx'_i)}_\infty \\
		                    & \leq 4 rB^2D^3 \norm{X_i - X'_i}_\infty.
	\end{align*}
	Here, we used the fact that $\norm{K_h}_\infty, \norm{Q_h}_\infty, \norm{V_h}_\infty \leq B$ and $\norm{X}_\infty \leq r$.
	Then, it holds that
	\begin{align*}
		\norm{V_h \bar X A_h - V_h \bar X A'_h }_\infty  & \leq BDr \norm{A_h - A'_h}_1              \\
		                                          & \leq 4B^3D^4r^2 \norm{X_i - X'_i}_\infty, \\
		\norm{V_h \bar X A'_h - V_h\bar X' A'_h }_\infty & \leq BD \norm{X - X'}_\infty.
	\end{align*}
	Since $B \geq 1, r\geq 1$, we have
	\begin{align*}
		\norm{g(X)_i - g(X')_i}_\infty & \leq \qty(1 + 4B^3D^4r^2 + BD)\norm{X - X'}_\infty \\
		                               & \leq 6HB^3D^4r^2 \norm{X - X'}_\infty,
	\end{align*}
	which completes the proof.
\end{proof}

\begin{lemma}\label{lem:norm}
	Let $f \in \Psi(L, W, S, B)$ and $g \in \mathcal{A}(U, D, H, B)$ for $B \geq 1$.
	For any $r \geq 1$, $X~(\norm{X}_\infty \leq r)$, and $x~(\norm{x}_\infty \leq r)$, we have
	\begin{align*}
		\norm{f(x)}_\infty & \leq (2BW)^L r \leq (6HDBW)^L r, \\
		\norm{g(X)}_\infty & \leq 2HBDr \leq (6HDBW)^L r.
	\end{align*}
\end{lemma}
\begin{proof}
	For the FNN layer $f$,
	since $\norm{\eta(z)}_\infty \leq \norm{z}_\infty$ and
	\begin{align*}
		\norm{Wz + b}_{\infty} & \leq BWr' + B \\
		                       & \leq 2BWr'
	\end{align*}
	for $\norm{z}_\infty \leq r'~(r' \geq 1)$, we have
	\begin{align*}
		\norm{f(X)}_\infty & \leq (2BW)^Lr
	\end{align*}
	by induction.

	For attention layer $g$, define $\bar X$ and $A_h$ as in the proof of Lemma~\ref{lem:lipschitz}. Then, we have
	\begin{align*}
		\norm{g_i(X)}_\infty & \leq \norm{x_i}_{\infty} + \sum_{i=1}^{H}\norm{V\bar X A_h}_\infty \\
		                     & \leq \norm{X}_\infty + HBD\norm{X}_{\infty}\norm{A_h}_1            \\
		                     & \leq 2HBDr,
	\end{align*}
	since $\norm{A_h}_1 = 1$.
	This completes the proof.
\end{proof}

\begin{lemma}\label{lem:diff}
	Define $f, \tilde f \in \Psi(L, W, B, S)$ by
	\begin{align*}
		f(x)        & := (A_L \cdot + b_L) \circ \dots \circ (A_1 x + b_1)                              \\
		\tilde f(x) & := (\tilde A_L \cdot + \tilde b_L) \circ \dots \circ (\tilde A_1 x + \tilde b_1),
	\end{align*}
	where $\norm{A_i - \tilde A_i}_\infty \leq \delta, \norm{b_i - \tilde b_i}_\infty \leq \delta$ for a given $\delta > 0$.
	For any $r \geq 1$ and $x~(\norm{x}_\infty \leq r)$, we have
	\begin{align*}
		\norm{f(x) - \tilde f(x)}_\infty & \leq 2(2BW)^L \delta r \leq (6HDBW)^{4L} \delta r^3.
	\end{align*}
	In addition, define $g, \hat g \in \mathcal{A}(U, D, H, B, S)$ by
	\begin{align*}
		g(X)_i        & := x_i + \sum_{i=1}^H V_h X[-U:U]\softmax((K_hX[-U:U])^\top(Q_h x_i))                             \\
		\tilde g(X)_i & := x_i + \sum_{i=1}^H \tilde V_h X[-U:U] \softmax((\tilde K_h X[-U:U])^\top(\tilde Q_h x_i)),
	\end{align*}
	where $\norm{K_h - \tilde K_h}_\infty \leq \delta$, $\norm{Q_h - \tilde Q_h}_\infty \leq \delta$, and $\norm{V_h - \tilde V_h}_\infty \leq \delta$ for a given $\delta > 0$.
	For any $r \geq 1$ and $X~(\norm{X}_\infty \leq r)$, we have
	\begin{align*}
		\norm{g(X) - \tilde g(X)}_\infty & \leq 5HB^2D^4r^3\delta \leq (6HDBW)^{4L} \delta r^3.
	\end{align*}
\end{lemma}
\begin{proof}
	By the same argument as Lemma 3 in~\citet{suzuki_adaptivity_2019}, we have
	\begin{align*}
		\norm{f(x) - \tilde f(x)}_{\infty} & \leq 2(2BW)^L r\delta.
	\end{align*}
	Define $\bar X, A_h, \tilde A_h$ as in Lemma~\ref{lem:lipschitz}. From Lemma~\ref{lem:softmax-lipschitz}, we have
	\begin{align*}
		\norm{A_h - \tilde A_h}_1 & \leq 2\norm{(K_h\bar X)^\top(Q_h x_i) - (\tilde K_h\bar X)^\top(\tilde Q_h x_i)}_\infty                                                                                      \\
		                          & \leq 2\norm{(K_h\bar X)^\top(Q_h x_i) - (K_h\bar X)^\top(\tilde Q_h x_i)}_\infty + 2\norm{(K_h\bar X)^\top(\tilde Q_h x_i) - (\tilde K_h\bar X)^\top(\tilde Q_h x_i)}_\infty \\
		                          & \leq 4(BD^3r^2\delta).
	\end{align*}
	Therefore, it holds that
	\begin{align*}
		\norm{g(X)_i - \tilde g(X)_i}_\infty & \leq \sum_{i=1}^{H} \norm{V_h \bar X A_h - \tilde V_h \bar X \tilde A_h}_\infty                                      \\
		                                     & \leq \sum_{i=1}^{H} \norm{V_h \bar X (A_h - \tilde A_h)}_\infty + \norm{(V_h - \tilde V_h) \bar X \tilde A_h}_\infty \\
		                                     & \leq \sum_{i=1}^{H} \norm{V_h \bar X}_\infty \norm{A_h - \tilde A_h}_1 + \norm{(V_h - \tilde V_h) \bar X}_\infty     \\
		                                     & \leq H(4(B^2D^4r^3\delta)+ D\delta r)                                                                                \\
		                                     & \leq 5HB^2D^4r^3\delta  ,
	\end{align*}
	which completes the proof.
\end{proof}

\begin{lemma}\label{lem:2a}
	Assume that positive and monotonically non-decreasing sequences $\bar a = \qty{\bar a_i}_{i=1}^\infty$ and $\bar a' = \qty{\bar a'_i}_{i=1}^\infty$ satisfies
	$\bar a_1 \geq 1 = \bar a'_1$ and
	\begin{align*}
		\prod_{i=2}^{\infty}\frac{1}{1-2^{-(\bar a_i - \bar a_1)}}       & < \infty, \\
		\prod_{i=2}^{\infty}\frac{1}{1-2^{-\beta(\bar a_i - \bar a'_i)}} & < \infty,
	\end{align*}
	for a positive constant $\beta$. Then, we have
	\begin{align*}
		\sum_{s \in \N_0^\infty :\ev{\bar a', s} \geq T} 2^{-\beta \ev{\bar a, s}} & \leq (1 - 2^{-\beta})^{-1}\qty(\prod_{i=2}^{\infty}\frac{1}{1-2^{-\beta(\bar a_i - \bar a'_i)}})2^{-\beta T}, \\
		\sum_{s \in \N_0^\infty :\ev{\bar a, s} < T}2^s                            & \leq 8\qty(\prod_{i=2}^{\infty}\frac{1}{1-2^{-(\bar a_i - \bar a_1)}})2^T.
	\end{align*}
\end{lemma}
\begin{proof}
	See Lemma 18 in~\citet{okumoto_learnability_2022}.
\end{proof}

\section{Approximation ability of FNN}
In this section, we show that an FNN can approximate (piecewise) $\gamma$-smooth function if important features are extracted properly.
For a general $\gamma$-smooth function class, we have the following approximation error bound.
\begin{lemma}\label{lem:gamma-smooth}
	For $\gamma: \N_0^{d\times \infty} \to \R_{>0}$,
	let
	\begin{align*}
		G(T, \gamma)        & := \sum_{s \in \N_0^{d\times \infty}: \gamma(s) < T} 2^s,                                     \\
		f_{\max}(T, \gamma) & := \max_{s \in \N_0^{d\times \infty}:\gamma(s) < T} \max_{i \in [d], j \in \Z} s_{ij}, \\
		I(T, \gamma)        & := \qty{(i, j) \mid \exists s \in \N_0^{d\times \infty}, s_{ij} \neq 0, \gamma(s) < T},       \\
		d_{\max}(T, \gamma) & := \abs{I(T, \gamma)}.
	\end{align*}
        
	Assume that $\gamma'$ satisfies $\gamma'(s) < \gamma(s)$ and the target function $f \in \mathcal{F}_{p, \theta}^\gamma([0, 1]^{d\times \infty}) (p \geq 2, \theta \geq 1)$
	satisfies $\norm{f}_{\infty} \leq R$ for a constant $R > 0$.

	For given $T > 0$, let
	\begin{align*}
		(d_{\max}, f_{\max}, G) & := \begin{cases}
			(d_{\max}(T, \gamma), f_{\max}(T, \gamma), G(T, \gamma)),    & \quad(\theta = 1), \\
			(d_{\max}(T, \gamma'), f_{\max}(T, \gamma'), G(T, \gamma')), & \quad(\theta > 1),
		\end{cases}
	\end{align*}
	and
	\begin{align}
            \begin{split}
		L & := 2K \max\qty{d_{\max}^2, T^2, (\log G)^2, \log f_{\max}},                \\
		W & := 21d_{\max}G,                                                            \\
		S & := 1764 Kd_{\max}^2 \max\qty{d_{\max}^2, T^2, (\log G)^2, \log f_{\max}}G, \\
		B & := 2^{d_{\max}/2}K',
            \end{split} \label{eq:hyper-parameters-fnn}
	\end{align}
	for positive constants $K, K'$ which depend on only $R$.
	Then, there exists an FNN $\hat f_T \in \Psi(L, W, B, S)$ such that
	\begin{align*}
		\norm{f - \hat f_T \circ \Gamma}_{2, P_X}           & \lesssim \begin{cases}
			2^{-T} \norm{f}_{\mathcal{F}_{p, \theta}^\gamma},                                                                                                 & \quad(\theta = 1), \\
			\qty(\sum_{T \leq \gamma(s)} 2^{\frac{\theta}{\theta - 1}(\gamma'(s) - \gamma(s))})^{1-1/\theta} 2^{-T}\norm{f}_{\mathcal{F}_{p, \theta}^\gamma}, & \quad(\theta > 1),
		\end{cases} 
	\end{align*}
	where $\Gamma:\R^{d \times \infty} \to \R^{d_{\max}}$ is a feature extractor defined by
	\begin{align}
		\begin{split}
			\Gamma(X) := [X_{i_1, j_1}, \dots, X_{i_{d_{\max}}, j_{d_{\max}}}]
		\end{split}\label{eq:Gamma}
	\end{align}
	for $X \in \R^{d \times \infty}$, and $I(T, \gamma) = \qty{(i_1, j_1), \dots, (i_{d_{\max}}, j_{d_{\max}})}$.
\end{lemma}
\begin{proof}
	Define $f_T$ by $f_T = \sum_{\gamma(s) < T} \delta_s(f)$ for $\theta = 1$ and $f_T = \sum_{\gamma'(s) < T} \delta_s(f)$ for $\theta > 1$.
	Then, for a given $\hat f \in \Psi(L, W, S, B)$, we have
	\begin{align*}
		\norm{f - \hat f \circ \Gamma}_{2, P_X} & \leq \norm{f - f_T}_{2, P_X} + \norm{f_T - \hat f \circ \Gamma}_{2, P_X}.
	\end{align*}

	First, we evaluate $\norm{f - f_T}_{2, P_X}$.
	By the Cauchy-Schwartz inequality, we have, for $p > 2$,
	\begin{align*}
		\norm{\delta_s}_{2, P_X}^2 & = \int \abs{\delta_s(f)}^2 \dd{P_X}                                                          \\
		                           & \leq \qty(\int \qty(\abs{\delta_s(f)}^2)^{p/2} \dd{P_X})^{2/p} \qty(\int 1 \dd{P_X})^{1-2/p} \\
		                           & = \norm{\delta_s(f)}_{p, P_X}^2.
	\end{align*}
	Therefore, for any $p \geq 2$, it holds that
	\begin{align*}
		\norm{\delta_s(f)}_{2, P_X} & \leq \norm{\delta_s(f)}_{p, P_X}.
	\end{align*}

	In the case of $\theta = 1$, we have
	\begin{align*}
		\norm{f - f_T}_{2, P_X} & = \norm{\sum_{\gamma(s) \geq T} \delta_s(f)}_{2, P_X}                             \\
		                        & \leq \sum_{\gamma(s) \geq T} \norm{\delta_s(f)}_{2, P_X}                          \\
		                        & \leq 	\sum_{\gamma(s) \geq T} \norm{\delta_s(f)}_{p, P_X}                          \\
		                        & = 	\sum_{\gamma(s) \geq T} 2^{\gamma(s)} 2^{-\gamma(s)}\norm{\delta_s(f)}_{p, P_X} \\
		                        & \leq 	2^{-T}\sum_{\gamma(s) \geq T} 2^{\gamma(s)} \norm{\delta_s(f)}_{p, P_X}      \\
		                        & \leq 2^{-T} \norm{f}_{\mathcal{F}_{p, \theta}^\gamma}.
	\end{align*}

	In the case of $\theta > 1$, we have
	\begin{align*}
		\norm{f - f_T}_{2, P_X} & \leq 	\sum_{\gamma'(s) \geq T} \norm{\delta_s(f)}_{p, P_X}                                                                                                                                                     \\
		                        & = 	\sum_{\gamma'(s) \geq T} 2^{-\gamma'(s)} 2^{\gamma'(s) - \gamma(s)} 2^{\gamma(s)}\norm{\delta_s(f)}_{p, P_X}                                                                                                \\
		                        & \leq 	2^{-T} \sum_{\gamma'(s) \geq T} 2^{\gamma'(s) - \gamma(s)} 2^{\gamma(s)}\norm{\delta_s(f)}_{p, P_X}                                                                                                      \\
		                        & \leq 	2^{-T} \qty(\sum_{\gamma'(s) \geq T} 2^{\frac{\theta}{\theta - 1}(\gamma'(s) - \gamma(s))})^{1-1/\theta} \qty(\sum_{\gamma'(s) \geq T} 2^{\theta\gamma(s)}\norm{\delta_s(f)}_{p, P_X}^\theta)^{1/\theta} \\
		                        & \leq  2^{-T} \qty(\sum_{\gamma'(s) \geq T} 2^{\frac{\theta}{\theta - 1}(\gamma'(s) - \gamma(s))})^{1-1/\theta} \norm{f}_{\mathcal{F}_{p, \theta}^\gamma}.
	\end{align*}
	For the third inequality, we used H\"{o}lder inequality.
	Combining these two cases, we obtain
	\begin{align}
		\norm{f - f_T}_{2, P_X} & \leq \begin{cases}
			2^{-T} \norm{f}_{p, \theta}^\gamma                                                                                                   & \quad(\theta = 1), \\
			\qty(\sum_{\gamma'(s) \geq T} 2^{\frac{\theta}{\theta - 1}(\gamma'(s) - \gamma(s))})^{1-1/\theta} 2^{-T} \norm{f}_{p, \theta}^\gamma & \quad(\theta > 1).
		\end{cases}\label{eq:f-T}
	\end{align}

	From Lemma 17 in~\cite{okumoto_learnability_2022}, there exists an FNN $\hat f_T \in \Psi(L, W, S, B)$ such that
	\begin{align}
		\norm{f_T - \hat f_T \circ \Gamma}_{\infty} & \leq 2^{-T}. \label{eq:hat-f-T}
	\end{align}

	Combining Eq.~\eqref{eq:f-T} and~\eqref{eq:hat-f-T}, we have
	\begin{align*}
		\norm{f - \hat f_T \circ \Gamma}_{2, P_X} & \lesssim \begin{cases}
			2^{-T} \norm{f}_{p, \theta}^\gamma                                                                                                   & \quad(\theta = 1), \\
			\qty(\sum_{\gamma'(s) \geq T} 2^{\frac{\theta}{\theta - 1}(\gamma'(s) - \gamma(s))})^{1-1/\theta} 2^{-T} \norm{f}_{p, \theta}^\gamma & \quad(\theta > 1).
		\end{cases}
	\end{align*}
        This completes the proof.
\end{proof}

In addition, for a general piecewise $\gamma$-smooth function class, we have the following approximation error bound.
\begin{lemma}\label{lem:piecewise-gamma-smooth}
	For $\gamma: \N_0^{d\times (2V+1)} \to \R_{>0}$,
	let
	\begin{align*}
		G(T, \gamma)        & := \sum_{s \in \N_0^{d\times (2V+1)}: \gamma(s) < T} 2^s,                                     \\
		f_{\max}(T, \gamma) & := \max_{s \in \N_0^{d\times (2V+1)}:\gamma(s) < T} \max_{i \in [d], j \in [2V+1]} s_{ij}, \\
		I(T, \gamma)        & := \qty{(i, j) \mid s \in \N_0^{d\times (2V+1)}, s_{ij} \neq 0, \gamma(s) < T},       \\
		d_{\max}(T, \gamma) & := \abs{I(T, \gamma)}.
	\end{align*}
        
	Assume that $\gamma'$ satisfies $\gamma'(s) < \gamma(s)$ and the target function $f \in \mathcal{P}_{p, \theta}^\gamma (p \geq 2, \theta \geq 1)$
	satisfies $\norm{f}_{\infty} \leq R$ for a constant $R > 0$.

	For given $T > 0$, let
	\begin{align*}
		(d_{\max}, f_{\max}, G) & := \begin{cases}
			(d_{\max}(T, \gamma), f_{\max}(T, \gamma), G(T, \gamma)),    & \quad(\theta = 1), \\
			(d_{\max}(T, \gamma'), f_{\max}(T, \gamma'), G(T, \gamma')), & \quad(\theta > 1),
		\end{cases}
	\end{align*}
	and define $L, W, S, B$ by Eq.~\eqref{eq:hyper-parameters-fnn}.
	Then, there exists an FNN $\hat f_T \in \Psi(L, W, B, S)$ such that
	\begin{align*}
		\norm{f - \hat f_T \circ \Gamma \circ \Pi}_{2, P_X} & \lesssim \begin{cases}
			2^{-T} \norm{f}_{\mathcal{P}_{p, \theta}^\gamma},                                                                                                 & \quad(\theta = 1), \\
			\qty(\sum_{T \leq \gamma(s)} 2^{\frac{\theta}{\theta - 1}(\gamma'(s) - \gamma(s))})^{1-1/\theta} 2^{-T}\norm{f}_{\mathcal{P}_{p, \theta}^\gamma}, & \quad(\theta > 1),
		\end{cases}
	\end{align*}
	where $\Gamma:\R^{d \times (2V+1)} \to \R^{d_{\max}}$ is a feature extractor defined by
	\begin{align}
		\begin{split}
			\Gamma(X) := [X_{i_1, j_1}, \dots, X_{i_{d_{\max}}, j_{d_{\max}}}]
		\end{split}\label{eq:piecewise-Gamma}
	\end{align}
	for $X \in \R^{d \times (2V+1)}$, and $I(T, \gamma) = \qty{(i_1, j_1), \dots, (i_{d_{\max}}, j_{d_{\max}})}$.
\end{lemma}
\begin{proof}
        From the definition of $\mathcal{P}_{p, \theta}^\gamma$, there exist $f' \in \mathcal{F}_{p, \theta}^\gamma([0, 1]^{d\times [2V+1]})$ such that $f = f' \circ \Pi$.
        Define $f_T$ by $f_T := \sum_{\gamma(s) < T} \delta_s(f')$ for $\theta = 1$ and $f_T := \sum_{\gamma'(s) < T} \delta_s(f')$ for $\theta > 1$.
	By the same argument as in the case of $\gamma$-smoothness, we have
	\begin{align*}
		\norm{f - f_T \circ \Pi}_{2, P_X} & \leq 2^{-T}\norm{f}_{\mathcal{P}_{p, \theta}^\gamma}
	\end{align*}
	for $\theta = 1$, and
	\begin{align*}
		\norm{f - f_T \circ \Pi}_{2, P_X} & \leq 2^{-T}\qty(\sum_{\gamma'(s) \geq T} 2^{\frac{\theta}{\theta - 1}(\gamma'(s) - \gamma(s))})^{1-1/\theta}\norm{f}_{\mathcal{P}_{p, \theta}^\gamma}
	\end{align*}
	for $\theta > 1$. By the same argument as in Lemma 17 in~\cite{okumoto_learnability_2022}, there exists an FNN $\hat f_T \in \Psi(L, W, S, B)$ such that
	\begin{align*}
		\norm{f_T - \hat f_T \circ \Gamma}_{\infty} & \leq 2^{-T}.
	\end{align*}
	This implies
	\begin{align*}
		\norm{f_T \circ \Pi - \hat f_T \circ \Gamma \circ \Pi}_{\infty} & \leq 2^{-T}.
	\end{align*}
	Therefore, we have
	\begin{align*}
		\norm{f - \hat f_T \circ \Gamma \circ \Pi}_{2, P_X} & \leq \norm{f - f_T \circ \Pi}_{2, P_X} + \norm{f_T \circ \Pi - \hat f_T \circ \Gamma \circ \Pi}_{\infty} \\
		                                                    & \lesssim \begin{cases}
			2^{-T} \norm{f}_{\mathcal{P}_{p, \theta}^\gamma},                                                                                                 & \quad(\theta = 1), \\
			\qty(\sum_{T \leq \gamma(s)} 2^{\frac{\theta}{\theta - 1}(\gamma'(s) - \gamma(s))})^{1-1/\theta} 2^{-T}\norm{f}_{\mathcal{P}_{p, \theta}^\gamma}, & \quad(\theta > 1),
		\end{cases}
	\end{align*}
	which completes the proof.
\end{proof}

Next, we evaluate the approximation ability of FNN for the (piecewise) mixed and anisotropic smoothness when the smoothness parameter $a$ satisfies $\norm{a}_{wl^\alpha} \leq 1$.
\begin{theorem}\label{thm:approximation-fnn}
	Suppose that the target functions $f \in \mathcal{F}_{p, \theta}^\gamma$ and $g \in \mathcal{P}_{p, \theta}^\gamma$ satisfy $\norm{f}_\infty \leq R$ and $\norm{g}_\infty \leq R$,
	where $R > 0$ and $\gamma$ is the mixed or anisotropic smoothness
	and the smoothness parameter $a$ satisfies $\norm{a}_{wl^\alpha} \leq 1$.
	For any $T > 0$, there exist FNNs $\hat f_T, \hat g_T \in \Psi(L, W, S, B)$ such that
	\begin{align*}
		\norm{\hat f_T \circ \Gamma - f}_{2, P_X}           & \lesssim 2^{-T}, \\
            \norm{\hat g_T \circ \Gamma \circ \Pi - g}_{2, P_X} & \lesssim 2^{-T},
	\end{align*}
	where
	\begin{align*}
		L & \sim \max\qty{T^{2/\alpha}, T^2}, W \sim T^{1/\alpha}2^{T / a^\dagger},                  \\
		S & \sim T^{2/\alpha}\max\qty{T^{2/\alpha}, T^2}2^{T / a^\dagger}, \log B \sim T^{1/\alpha}.
	\end{align*}
\end{theorem}
The result for the fixed smoothness case is similar to Theorem 7 in~\citet{okumoto_learnability_2022},
but we omit the case $1 \leq p\leq 2$ since we relax the assumption $\norm{\frac{\dd{P_X}}{\dd{\lambda}}}_\infty < \infty$, which is imposed in the previous work.

\begin{proof}
	We show the result for the mixed smoothness and anisotropic smoothness separately.
	\paragraph{Mixed smoothness}
	Let us consider the case $\theta = 1$.
	Since $\norm{a}_{wl^\alpha} \leq 1$, $\bar a_j \geq j^{\alpha}$ for any $j \in \N$.
	Therefore, we see that $d_{\max} \sim T^{1/\alpha}$ and $f_{\max} \sim T$.
	In addition, from Lemma~\ref{lem:2a}, we have
	\begin{align*}
		G(T, \gamma) & = \sum_{\ev{a/\bar a_1, s} < T/\bar a_1} 2^s \lesssim 2^{T/{\bar a_1}}.
	\end{align*}
	Therefore, from Lemmas~\ref{lem:gamma-smooth} and~\ref{lem:piecewise-gamma-smooth}, there exists $\hat f_T, \hat g_T \in \Psi(L, W, S, B)$ such that
	\begin{align*}
		\norm{\hat f_T \circ \Gamma - f}_{2, P_X}           & \lesssim 2^{-T}, \\
		\norm{\hat g_T \circ \Gamma \circ \Pi - g}_{2, P_X} & \lesssim 2^{-T},
	\end{align*}
        where $L, W, S, B$ is defined in Eq.~\eqref{eq:hyper-parameters-fnn}.

	In the case of $\theta > 1$, let $\bar a'_1 = a_1/2$ and $\bar a'_i = \bar a_i/u~(i \geq 2)$
	for $2 < u < 2 + \frac{2\delta}{\bar a_1}$, where $\delta = \bar a_1 - \bar a_2 > 0$.
	Then, $\bar a'$ is a positive monotonically increasing sequence.
	Since $\bar a_i \geq i^\alpha$, we have, for any $c > 0$,
	\begin{align*}
		\prod_{i=2}^{\infty}\frac{1}{1-2^{-c\qty(\bar a_i/\bar a_1 - 2\bar a'_i/\bar a_1)}} & < \infty.
	\end{align*}
	Therefore, by adjusting the scale of $\bar a, \bar a'$, Lemma~\ref{lem:2a} yields that
	\begin{align*}
		\sum_{\ev{a', s} \geq T} 2^{\frac{\theta}{\theta - 1}(\gamma'(s) - \gamma(s))} & \lesssim 2^{-\frac{\theta}{\theta - 1}T},   \\
		G(T, \gamma')                                                                  & \lesssim 2^{T/\bar a'_1} = 2^{2T/\bar a_1}.
	\end{align*}
	Therefore, from Lemma~\ref{lem:gamma-smooth}, there exists $\hat f_T, \hat g_T \in \Psi(L, W, S, B)$ such that
	\begin{align*}
		\norm{\hat f_T \circ \Gamma - f}_{2, P_X}           & \lesssim \qty(2^{-\frac{\theta}{\theta - 1}T})^{1 - 1/\theta} 2^{-T} = 2^{-2T}, \\
		\norm{\hat g_T \circ \Gamma \circ \Pi - g}_{2, P_X} & \lesssim \qty(2^{-\frac{\theta}{\theta - 1}T})^{1 - 1/\theta} 2^{-T} = 2^{-2T}.
	\end{align*}
	By replacing $2T \leftarrow T$, we obtain the result.

	\paragraph{Anisotropic smoothness}
	For $s \in \N_0^\infty$, define $\bar s$ as the sequence rearranged in the same way as $\bar a$.
	Since $\gamma(s) \leq T$ is equivalent to $\bar s_i \leq T/\bar a_i$ for any $i\in \N$, we have
	\begin{align*}
		\sum_{\gamma(s) < T} 2^s\leq \prod_{i=1}^\infty \qty(\sum_{\bar s_i = 0}^{\lceil T/\bar a_i \rceil}2^{\bar s_i}) & \leq 2^{\sum_{i=1}^\infty \lceil T/\bar a_i \rceil} \sim 2^{T/\tilde a}.
	\end{align*}
	Then, by the same argument as in the case of mixed smoothness, we obtain the result.
\end{proof}

\section{Proof of Theorem~\ref{thm:approximation}}\label{sec:proof-approximation}
First, we construct the embedding layer $\enc_P$.
Let $D$ be $d + d_{\max} + 2$ and the embedding matrix $E \in \R^{D \times d}$ be the matrix
such that $Ex = [x_1, \dots, x_d, 0, \dots, 0]^\top \in \R^D$ for any $x \in \R^d$. Note that $\norm{E}_\infty = 1$.
Define $Z^{(m)}$ by
\begin{align}
	\begin{split}
		Z^{(0)} & = \qty{z^{(0)}_j}_{j=-\infty}^\infty := \enc_P(X),                                                                   \\
		Z^{(m)} & = \qty{z^{(m)}_j}_{j=-\infty}^\infty := g_m \circ f_{m-1} \circ g_{m-1} \circ \dots f_1 \circ g_1 \circ \enc_P(X)\quad(j=1, \dots, M).
	\end{split}\label{eq:Zm}
\end{align}
Then, the embedded token $z^{(0)}_j$ is given by $[X_{1, j}, \dots, X_{d, j}, 0, \dots, 0, \cos(j\phi), \sin(j\phi)]^\top$.

Next, we construct the attention layer $g_1$.
Here, for each self-attention head, define the parameters $Q_h, K_h, V_h (h=1, \dots, d_{\max})$ by
\begin{align*}
	Q_h & := \chi \mqty[
	0   & \dots                 & 0 & \cos(j_h\phi) & -\sin(j_h\phi) \\
	0   & \dots                 & 0 & \sin(j_h\phi) & \cos(j_h\phi)
	],                                                               \\
	K_h & := \mqty[
	0   & \dots                 & 0 & 1             & 0              \\
	0   & \dots                 & 0 & 0             & 1
	],                                                               \\
	V_h & := \delta_{d+h, i_h},
\end{align*}
where $ \chi$ is a constant and $I(T, \gamma) = \qty{(i_h, j_h)}_{h=1}^{d_{\max}}$.
Then, the key, query, and value vectors for a given head $h$ are given by
\begin{align*}
	q_i & := Q_h z^{(0)}_i = \chi \qty[\cos((i + j_h)\phi), \sin((i+j_h)\phi)]^\top, \\
	k_i & := K_h z^{(0)}_i = \qty[\cos(i\phi), \sin(i\phi)]^\top,                    \\
	v_i & := V_h z^{(0)}_i = X_{i_h, i}e_{d+h}.
\end{align*}

From the assumption $a_{ij} = \Omega(\log(\abs{j} + 1))$,
there exists a window size $U$ such that $\log U  \sim T$ and $j \leq U$ if $a_{ij} \leq T$.
That is, $j \in [-U,U]$ if $(i, j) \in I(T, \gamma)$.
Define $\tilde Z = [\dots, \tilde z_0, \dots, \tilde z_i, \dots]$ by
\begin{align*}
	\tilde z_j & := z^{(0)}_j + \sum_{h=1}^H V_h Z^{(0)}[j-U:j+U] e_{j_h} = z^{(0)}_j + \sum_{h=1}^H V_h z^{(0)}_{j_h}.
\end{align*}
Intuitively, $\tilde z_j$ corresponds to $z^{(1)}_j$ in the situation where the softmax operation in the attention layer is replaced by  \textit{hardmax} operation.
Then, we have
\begin{align*}
	\tilde z_j & = [X_{1, j}, \dots, X_{d, j}, X_{i_1, j + j_1}, \dots, X_{i_{d_{\max}}, j + j_{d_{\max}}}, \cos(j\phi), \sin(j\phi)]^\top.
\end{align*}
Note that $\tilde z_j$ contains important features $(X_{i_1, j + j_1}, \dots, X_{i_{d_{\max}}, j + j_{d_{\max}}})$ for $j$-th output.
For a little while, we focus on the $0$-th token. Let $s^{(h)} := (K_h Z^{(0)}[-U:U])^\top (Q_h z^{(0)}_0)$.
Then, we have
\begin{align*}
	s^{(h)}_i & = k_{i}^\top q_0 = \chi(\cos(i\phi)\cos((j_h\phi)) + \sin(i\phi) \sin(j_h\phi))= \chi \cos((j_h - i)\phi),
\end{align*}
which implies $s^{(h)}_{j_h} \geq s^{(h)}_j + \chi/U^2(\forall j \neq j_h)$.
From Lemma~\ref{lem:softmax}, we have
\begin{align*}
	\norm{z^{(1)}_0 - \tilde z_0}_{\infty} & = \norm{\sum_{h=1}^H V_h Z^{(0)}[-U:U] \qty(e_{j_h} - \softmax(s^{(h)}))}_\infty            \\
	                                       & \leq \sum_{h=1}^H \norm{V_hZ^{(0)}[-U:U]}_\infty \norm{\qty(e_{j_h} - \softmax(s^{(h)}))}_1 \\
	                                       & \leq 2HUe^{-\chi/U^2},
\end{align*}
where the last inequality holds because $\norm{V^{(h)} z^{(0)}_j}_\infty = \norm{Z^{(0)}_{i_h, j}e_{d + h}}_\infty \leq 1$ for any $j \in [-U: U]$.
Similarly, we have $\norm{z^{(1)}_j - \tilde z_j}_{\infty}\leq 2HUe^{-\chi/U^2}$ for any $j \in \Z$.
Let $C \in \R^{d_{\max} \times D}$ be the matrix such that for any $x \in \R^D$, $Cx = [x_{d+1}, \dots, x_{d+H}]^\top$.
Then, we have $C\tilde z_j = [X_{i_1, j + j_1}, \dots, X_{i_H, j + j_H}]^\top = \Gamma \circ \Sigma_{j}(X)$, where $\Gamma$ is defined in Eq.~\eqref{eq:Gamma}.

Next, we construct the FNN layer $f_1$.
From Theorem~\ref{thm:approximation-fnn}, there exists an FNN $\hat f \in \Psi(L, W, S, B')$ such that
\begin{align*}
	\norm{\hat f \circ \Gamma - F^\circ_0}_{2, P_X} & \lesssim 2^{-T}.
\end{align*}
where $\log B' \sim T^{1/\alpha}$.
From the shift-equivariance of $F^\circ$, we have
\begin{align}
	F^\circ_i(X) = F^\circ_0(\Sigma_{i}(X)),
\end{align}
Therefore, we have
\begin{align*}
	\norm{\hat f \circ \Gamma \circ \Sigma_{i} - F^\circ_i}_{2, P_X} & = \qty(\int_{\Omega} \norm{\hat f \circ \Gamma \circ \Sigma_{i}(X) - F^\circ_i(X)}_2^2 \dd{P_X})^{1/2}         \\
	                                                                  & = \qty(\int_{\Omega} \norm{\hat f \circ \Gamma(\Sigma_{i}(X)) - F^\circ_0(\Sigma_{i}(X))}_2^2 \dd{P_X})^{1/2} \\
	                                                                  & = \qty(\int_{\Omega} \norm{\hat f \circ \Gamma(X) - F^\circ_0(X)}_2^2 \dd{P_X})^{1/2}                           \\
	                                                                  & \lesssim 2^{-T},
\end{align*}
for any $i\in \Z$. For the third equality, we used the shift-invariance of $P_X$.
Let $f_1 = \hat f \circ C$ and $\hat F = f_1 \circ g_1 \circ \enc_P$.
Since $f_1 \in \Psi(L, W, S, B')$ and $f_1$ is $(B'W)^{L}$-Lipschitz continuous with respect to $\norm{\cdot}_\infty$ norm,
we have, for $X \in \domain$,
\begin{align*}
	\abs{\hat f \circ \Gamma \circ \Sigma_{i}(X) - \hat F_i(X)} & = \abs{f_1(\tilde z_i) - f_1(z^{(1)}_i)} \\
	                                                             & \leq (B'W)^{L} \cdot 2HUe^{-\chi/U^2}.
\end{align*}
By putting $\chi = U^2 \log(2HU (B'W)^{L}2^{T})$, we have
\begin{align*}
	\abs{\hat f \circ \Gamma \circ \Sigma_{i}(X) - \hat F_i(X)} & \leq 2^{-T}.
\end{align*}
Therefore, it holds that
\begin{align*}
	\norm{F^\circ_i - \hat F_i}_{2, P_X} & \leq \norm{F^\circ_i - \hat f \circ \Gamma \circ \Sigma_{i}}_{2, P_X} + \norm{\hat f \circ \Gamma \circ \Sigma_{i} - \hat F_i}_{\infty} \\
	                                     & \lesssim 2^{-T}.
\end{align*}
The scaling factor $\chi$ can be evaluated as follows:
\begin{align*}
	\log \chi & = \log\qty[U^2 \log(2HU (B'W)^{L}2^{T})] \\
	          & \sim T.
\end{align*}
Therefore, $g_1 \in \mathcal{A}(U, D, H, B)$ and $\hat F \in \mathcal{T}(M, U, D, H, L, W, S, B)$, which completes the proof.\qed

\section{Proof of Theorem~\ref{thm:approximation-piecewise}}\label{sec:proof-approximation-piecewise}
For $T > 0$, define
\begin{align*}
	I_j(T, \gamma)      & := \qty{i \mid (i, j) \in I(T, \gamma)} = \qty{i^{(j)}_1, \dots, i^{(j)}_{\abs{I_j}}}, \\
	r_{\max}(T, \gamma) & := \max\qty{j \mid \abs{I_j(T, \gamma)} \neq 0}.
\end{align*}
Note that $r_{\max} \sim T^{1/\alpha}$ since $a_{ij} = \Omega(j^\alpha)$.
Let $\qty{u_i}_{i=1}^{2V+1}$ be $\varepsilon_1$-approximately orthonormal vectors such that
\begin{align*}
	\norm{u_i}_2 = 1, \\
	\ev{u_i, u_j} \close{\varepsilon_1} 1_{i = j},
\end{align*}
where $\varepsilon_1 > 0$ and $d' =O(\frac{\log V}{\varepsilon_1^2})$.
In addition, we define $u_i = u_{i \mod 2V +1}$ for $i \notin [2V+1]$.
Note that for any $i, j, k \in \N$ such that $i, j \in [k-V:k+V]$,
$\ev{u_i, u_j} \close{\varepsilon_1} 1_{i = j}$ holds.

Let $M = r_{\max} + 1$, $\phi = 2\pi / (2U_1 + 1)$, $D = d + d_{\max} + 2d' + 4$, $E = \sum_{i=1}^d \delta_{ii} \in \R^{D\times d}$, and
\begin{align*}
	p_i = [0, \dots, 0, 1, \cos(i\phi), \sin(i\phi), u_i^\top]^\top \in \R^d.
\end{align*}
Then, $\enc_P$ is defined as
\begin{align*}
	\enc_P(X)_i := Ex_i + p_i = [x_i^\top, 0, \dots, 0, 1, \cos(i\phi), \sin(i\phi), u_i^\top]^\top.
\end{align*}
Let $Z^{(0)} = \enc_P(X)$, $Z^{(m)} = f_m \circ g_m(Z^{(m-1)})~(m=1, \dots, M)$.
By the same argument as in Theorem~\ref{thm:approximation} with $T \sim \log 1 / \varepsilon_1$,
there exists an FNN $f_1$ and an attention layer $g_1$ such that 
\begin{align*}
	z_i^{(1)} & = x^{(1)}_i + y_i,                                                                  \\
	x_i^{(1)} & = \tilde x^{(1)}_i := [x_i^\top, 0, \dots, 0, u_i^\top, 0, \dots, 0]^\top,          \\
	y_i &:= [0, \dots, 0, \hat \mu_i(X), 1, 0, \dots, 0]^\top,\\
	\tilde y_i &:= [0, \dots, 0, \mu_i(X), 1, 0, \dots, 0]^\top,\\
        \norm{\hat \mu - \mu}_\infty &\lesssim \varepsilon_1.
\end{align*}
Note that Theorem~\ref{thm:approximation} holds for $\norm{\cdot}_{2, P_X}$ but it can be easily extended to $\norm{\cdot}_{\infty}$ sicne $p=\infty$ is assumed.
We also define $\tilde z^{(1)}_i = \tilde x^{(1)}_i + \tilde y_i$.

In the following, we fix $\lambda \in \Lambda$ and $X \in \Omega_\lambda$.
For fixed $\lambda$ and $X$, we denote $\pi_\lambda$ by $\pi$ and $\mu(X)_j$ by $\mu_j$ for simplicity.
Let $r_i(m) = \pi_{\lambda_i}(m)+i$, where $\Sigma_{i}(X) \in \Omega_{\lambda_i}$.
Since $\eta(x) - \eta(-x) = x$, there exists an FNN $f \in \Psi(2, 2d, 4d, 1)$ such that $f(x) = x$ for any $x \in \R^d$.
We set $f_m = f$ for $m = 2, \dots, M-1$.
In addition, we set the parameters for $i$-th heads of $g_m~(i=2, \dots H,~m=2, \dots, M)$ by zero matrix.
For the first head of $g_m$, we define the parameters $U_m, K_m, Q_m, V_m$ by $U_m = V$ and
\begin{align*}
	K_{m+1} & := \delta_{1, d + d_{\max} + d' + 1} + \sum_{i=1}^{d'} \delta_{i+1, D - 2d' + i},                                                          \\
	Q_{m+1} & := \chi\qty(\delta_{1, d + d_{\max} + d' + 2} - (2 + cr_{\max}^{-\beta})\sum_{i=1}^{d'} \delta_{i+1, D - d' + i}),                         \\
	V_{m+1} & := \sum_{i=1}^{\abs{I_{m}}} \delta_{i + \sum_{m' = 1}^{m-1}\abs{I_{{(m')}}}, i^{{m'})}_i} + \sum_{i=1}^{d'} \delta_{D-d'+i, D-2d'+1}.
\end{align*}
For $m=1, \dots, M-1$, let
\begin{align*}
	\tilde x_i^{(m+1)} & := \tilde x_i^{(m)} + V_{m+1}\tilde X^{(m)}[i-U_m:i+U_m] e_{r_i(m)}, \\
	\tilde z_i^{(m+1)} & := \tilde z_i^{(m)} + V_{m+1}\tilde Z^{(m)}[i-U_m:i+U_m] e_{r_i(m)}.
\end{align*}
Note that $\tilde z^{(m)}_i = \tilde x^{(m)}_i + \tilde y_i$ and $z^{(m)}_i = x^{(m)}_i + y_i$ since $V_{m+1} \tilde y_i = V_{m+1} y_i = 0$ for any $m$.
Since
\begin{align*}
	V_{m+1} \tilde x_i^{(m)} = [0, \dots, 0, X_{i^{(m)}_1, i}, \dots, X_{i^{(m)}_{\abs{I_{m}}}, i}, 0, \dots, 0, u_{i}^\top]^\top,
\end{align*}
we have
\begin{align*}
	\tilde x^{(m+1)}_i = [x_i^\top, X_{i^{(1)}_1, r_i(1)}, \dots, X_{i^{(m)}_{\abs{I_{m}}}, r_i(m)}, 0, \dots, 0, u_i^\top, w^{(m+1)}_i],
\end{align*}
where $w^{(m+1)}_i := \sum_{m'=1}^{m}u_{r_i(m')}$.
Let $C := \sum_{i=1}^{d_{\max}} \delta_{i, d+i} \in \R^{d_{\max}\times D}$.
Then,
\begin{align}
	C \tilde z^{(M)}_k = [X_{i^{(1)}_1, r_k(1)}, \dots, X_{i^{({r_{\max}})}_{\abs{I_{{r_{\max}}}}}, r_k(r_{\max})}]^\top = \Gamma \circ \Pi \circ \Sigma_{k},\label{eq:Cz}
\end{align}
since $[\Pi \circ \Sigma_{k}(X)]_{i, j} = [\Sigma_{k}(X)]_{i, \pi_{\lambda_k}(j)} = X_{i, \pi_{\lambda_k}(j)+k}$.

Next, we show that for any $\varepsilon_2 = O(3^{-r_{\max}}\varepsilon_1/D)$, $\norm{x^{(M)}_i - \tilde x^{(M)}_i}_{\infty} \leq 3^{r_{\max}}\varepsilon_2$ by setting $\chi$ appropriately.
For $m=1$, we have $\norm{x^{(m)}_i - \tilde x^{(m)}_i}_\infty = 0$.
Assume that $\norm{x^{(m)}_i - \tilde x^{(m)}_i}_\infty = 3^{m-1}\varepsilon_2$ for some $m = 1, \dots, M-1$.
The key and query vectors,
\begin{align*}
	k_i & := K_{m+1}z^{(m)}_i,~ \tilde k_i := K_{m+1} \tilde z^{(m)}_i = [\mu_i, u_i]^\top,                                 \\
	q_i & := Q_{m+1}z^{(m)}_i,~ \tilde q_i := Q_{m+1} \tilde z^{(m)}_i = \chi [1, -(2 + cr_{\max}^{-\beta})w^{(m)}_i]^\top,
\end{align*}
satisfy
\begin{align*}
	\norm{k_i - \tilde k_i}_2 & \lesssim \norm{z^{(m)}_i - \tilde z^{(m)}_i}_2 \leq \norm{y_i - \tilde y_i}_2 + D \norm{x^{(m)}_i - \tilde x^{(m)}_i}_\infty \lesssim (\varepsilon_1 + 3^{(m-1)}D \varepsilon_2) \sim \varepsilon_1, \\
	\norm{q_i - \tilde q_i}_2 & \lesssim \chi \norm{z^{(m)}_i - \tilde z^{(m)}_i}_2 \lesssim \chi \varepsilon_1,                                                                                         \\
	\norm{k_i}_2              & \lesssim 1,~\norm{\tilde q_i}_2 \lesssim \chi r_{\max}.
\end{align*}
Therefore, for $m' \in [2V+1]$, we have
\begin{align*}
	k_{r_i(m')}^\top q_i & \close{\chi \varepsilon_1} k_{r_i(m')}^\top \tilde q_i                                               \\
	                     & \close{\chi r_{\max}\varepsilon_1} \tilde k_{r_i(m')}^\top \tilde q_i                                \\
	                     & = \chi (\mu_{r_i(m')} - (2 + cr_{\max}^{-\beta}) u_{r_i(m')}^\top w^{(m)}_i)                         \\
	                     & \close{\chi r_{\max} \varepsilon_1} \chi (\mu_{r_i(m')} - (2 + cr_{\max}^{-\beta})1_{m' \in [m-1]}),
\end{align*}
since $u_{r_i(m')}^\top w^{(m)}_i = \sum_{k=1}^{m-1}u_{r_i(m')}^\top u_{r_i(k)} \close{r_{\max}\varepsilon_1} 1_{m' \in [m-1]}$.
For $m' < m$, we have
\begin{align*}
	k_{r_i(m)}^\top q_i - k_{r_i(m')}^\top q_i & \close{\chi r_{\max}\varepsilon_1} \chi (\mu_{r_i(m')} - \mu_{r_i(m)} + 2 + cr_{\max}^{-\beta}), \\
	                                           & \geq \chi cr_{\max}^{-\beta},
\end{align*}
since $\abs{\mu_i} \leq 1$.
For $m' > m$, we have
\begin{align*}
	k_{r_i(m)}^\top q_i - k_{r_i(m')}^\top q_i & \close{\chi r_{\max}\varepsilon_1} \chi (\mu_{r_i(m')} - \mu_{r_i(m)}), \\
	                                           & \geq \chi cr_{\max}^{-\beta}.
\end{align*}
This is because for $m' > m \leq r_{\max}$,
\begin{align*}
	[\mu \circ \Sigma_{i}(X)]_{\pi_{\lambda_i}(j_{m})} - [\mu \circ \Sigma_{i}(X)]_{\pi_{\lambda_i}(j_{m'})} \geq cm^{-\beta} \geq cr_{\max}^{-\beta},
\end{align*}
since $\mu$ is well-separated, and
\begin{align*}
	\mu_{r_i(m')} = [\mu  \circ \Sigma_{i}(X)]_{\pi_{\lambda_i}(j_{m'})},
\end{align*}
which yields
\begin{align*}
	\mu_{r_i(m)} - \mu_{r_i(m')} \geq cr_{\max}^{-\beta}.
\end{align*}
Therefore, we have, for any $m' \neq m$,
\begin{align*}
	k_{r_i(m)}^\top q_i - k_{r_i(m')}^\top q_i \geq \chi cr_{\max}^{-\beta} / 2,
\end{align*}
by letting $\varepsilon_1 \sim r_{\max}^{-\beta-1}$.

Let
\begin{align*}
	s_i & := \softmax([k_{i - V}^\top q_i, \dots, k_{i+V}^\top q_i]).
\end{align*}
From Lemma~\ref{lem:softmax}, we have
\begin{align*}
	\norm{s_i - e_{r_i(m)}}_1 & \leq (4V+2)e^{-\chi \frac{cr_{\max}^{-\beta}}{2}} \\
	                          & \leq \varepsilon_2 / r_{\max},
\end{align*}
by letting $\chi = \frac{2r_{\max}^{\beta}}{c} \log [{(4V+2)r_{\max}/\varepsilon_2}]$.
Therefore, it holds that
\begin{align*}
	\norm{x^{(m+1)}_i - \tilde x^{(m+1)}_i}_\infty & \leq \norm{x^{(m)}_i - \tilde x^{(m)}_i}_\infty + \norm{V_{m+1}\tilde X^{(m)}[i-V, i+V]e_{r_i(m)} - V_{m+1}X^{(m)}[i-V, i+V]s_i}_\infty,  \\
	                                               & \leq 3^{m-1}\varepsilon_2 + \norm{\tilde X^{(m)}}_\infty \norm{e_{r_i(m)} - s_i}_1 + \norm{X^{(m)} - \tilde X^{(m)}}_\infty \norm{s_i}_1, \\
	                                               & \leq 3^{m-1} \varepsilon_2 + \varepsilon_2 + 3^{m-1}\varepsilon_2 \leq 3^m \varepsilon_2,
\end{align*}
since $\norm{s_i}_1 = 1$ and $\norm{\tilde X^{(m)}}_\infty \leq r_{\max}$.
By induction, we obtain $\norm{x^{(M)}_i - \tilde x^{(M)}_i}_{\infty} \leq 3^{r_{\max}}\varepsilon_2$.

From Lemma~\ref{thm:approximation-fnn}, there exists $\hat f_T \in \Psi(L, W, S, B')$ such that $\log B' \sim T^{1/\alpha}$ and
\begin{align*}
	\norm{\hat f_T \circ \Gamma \circ \Pi - F^\circ_0}_{2, P_X} & \lesssim 2^{-T}.
\end{align*}
By the same argument as in Theorem~\ref{thm:approximation}, we have
\begin{align*}
	\norm{\hat f_T \circ \Gamma \circ \Pi \circ \Sigma_{i} - F^\circ_i}_{2, P_X} & = \norm{\hat f_T \circ \Gamma \circ \Pi - F^\circ_0}_{2, P_X} \\
	                                                                              & \lesssim 2^{-T}
\end{align*}
for any $i$.
Let $f_M = \hat f_T \circ C \in \Psi(L, W, S, B')$.
Then, Eq.~\eqref{eq:Cz} yields
\begin{align*}
	f_M(\tilde z^{(M)}_i) = \hat f_T \circ \Gamma \circ \Pi \circ \Sigma_{i}(X).
\end{align*}
Define $\hat F = f_M \circ g_M \circ \dots f_1 \circ g_1 \circ \enc_P$.
Since $f_M$ is $(B'W)^L$-Lipschitz continuous, we have
\begin{align*}
	\abs{\hat f_T\circ \Gamma \circ \Pi\circ \Sigma_{i}(X) - \hat F_i(X)} & \leq (B'W)^L \norm{z^{(M)}_i - \tilde z^{(M)}_i}_\infty \\
	                                                                       & \leq (B'W)^L 3^{r_{\max}}\varepsilon_2,
\end{align*}
for any $X \in \Omega$.
Therefore, by letting $\varepsilon_2 = 3^{-r_{\max}}(B'W)^{-L}2^{-T}$, we have
\begin{align*}
	\norm{F^\circ_i - \hat F_i}_{2, P_X} & \leq \norm{F^\circ_i - \hat f_T \circ \Gamma \circ \Pi \circ \Sigma_{i}}_{2, P_X} + \norm{ \hat f_T \circ \Gamma \circ \Pi \circ \Sigma_{i} - \hat F_i}_\infty \\
	                                     & \lesssim 2^{-T}.
\end{align*}
Here, we have
\begin{align*}
	\log \chi & = \log\qty(\frac{2r_{\max}^\beta}{c}\log ((4V+2)r_{\max}3^{r_{\max}}(B'W)^{L}2^T))\sim \log T + \log\log V, \\
	D         & = d + d_{\max} + 2d' + 2 \sim T^{2(\beta+1)/\alpha} \log V,\\
	H &\sim (\log 1/\varepsilon_1)^{1/\alpha} \sim (\log T)^{1/\alpha},\\
        \log U_1 &\sim \log (1 / \varepsilon_1) \sim \log T.
\end{align*}
Thus, $g_i \in \mathcal{A}(U_i, D, H, B)$, $f_i \in \Psi(L, W, S, B)$, and $\hat F \in \mathcal{T}(M, U, D, H, L, W, S, B)$, which completes the proof.

\section{Proof of Theorem~\ref{thm:excess-risk}}\label{sec:peoof-excess-risk}
To simplify the notation, let $F^\circ(X) \leftarrow F^\circ(X)[l:r]$, $\xi^{(i)}\leftarrow \xi^{(i)}[l:r]$, $Y^{(i)} \leftarrow Y^{(i)}[l:r]$, $l \leftarrow r - l + 1$, and $N \leftarrow \mathcal{N}(\mathcal{F}, \delta, \norm{\cdot}_\infty)$.
Define
\begin{align*}
	\hat R      & := \mathbb{E}\qty[\frac{1}{nl}\sum_{i=1}^n \norm{\hat F(X^{(i)}) - F^\circ(X^{(i)})}_2^2], \\
	\mathcal{D} & := \abs{\hat R - R(\hat F, F^\circ)}.
\end{align*}
Then, we have
\begin{align*}
	R(\hat F, F^\circ) & \leq \hat R + \mathcal{D}.
\end{align*}

First, we evaluate $\mathcal{D}$.
Let $G_\delta$ be a minimal $\delta$-covering of $\mathcal{F}$ in $L^\infty$ norm such that $\abs{G_\delta} = N$.
Then, there exists a random variable $J \in [N]$ such that $\norm{\hat F - F_J}_\infty \leq \delta$.
Define
\begin{align*}
	g_j(X, X') & = \frac{1}{l}\qty{\norm{F_j(X) - F^\circ (X)}_2^2 - \norm{F_j(X') - F^\circ(X')}_2^2},
\end{align*}
and we have
\begin{align}
	\abs{\norm{\hat F(X) - F^\circ (X)}_2^2 - \norm{F_J(X) - F^\circ (X)}_2^2} & = \ev{\hat F(X) - F_J(X), \hat F(X) + F_J(X) - 2 F^\circ(X)}             \notag  \\
	                                                                           & \leq \norm{\hat F(X) - F_J(X)}_2\norm{\hat F(X) + F_J(X) - 2F^\circ(X)}_2 \notag \\
	                                                                           & \leq 4lR\delta.\label{eq:4rd}
\end{align}
For the last inequality, we use $\norm{\hat F(X) - F_J(X)}_2 \leq \sqrt{l}\norm{\hat F(X) - F_J(X)}_\infty \leq \sqrt{l}\delta$ and
$\norm{F^\circ}_\infty \leq R, \norm{F}_\infty \leq R$ for any $F \in \mathcal{F}$.

Let $\tilde X^{(1)}, \dots, \tilde X^{(n)}$ be i.i.d. random variables independent of $(X^{(i)}, Y^{(i)})$.
Then,
\begin{align*}
	R(\hat F, F^\circ) & = \frac{1}{nl}\sum_{i=1}^n \mathbb{E}\qty[\norm{\hat F(\tilde X^{(i)}) - F^\circ(\tilde X^{(i)})}_2^2]
\end{align*}
holds and we have
\begin{align*}
	\mathcal{D} & = \abs{\frac{1}{nl}\sum_{i=1}^n \qty(\mathbb{E}\qty[\norm{\hat F(X^{(i)}) - F^\circ(X^{(i)})}_2^2] - \mathbb{E}\qty[\norm{\hat F(\tilde X^{(i)}) - F^\circ(\tilde X^{(i)})}_2^2])} \\
	            & \leq \frac{1}{nl} \mathbb{E}\qty[\abs{\sum_{i=1}^n\norm{\hat F(X^{(i)}) - F^\circ(X^{(i)})}_2^2 - \norm{\hat F(\tilde X^{(i)}) - F^\circ(\tilde X^{(i)})}_2^2}]                    \\
	            & \leq \frac{1}{n} \mathbb{E}\qty[\abs{\sum_{i=1}^ng_J(X^{(i)}, \tilde X^{(i)})}] + 8R\delta.
\end{align*}
Here, we used Eq.~\eqref{eq:4rd}.
Let $r_i = \max\qty{A, l^{-1/2}\norm{F_j - F^\circ}_2}$ and
\begin{align*}
	T & := \max_j{\sum_{i=1}^n \frac{g_j(X^{(i)}, \tilde X^{(i)})}{r_j}}
\end{align*}
for $A > 0$, which is determined later.
Then, we have
\begin{align}
	\mathcal{D} & \leq \frac{1}{n} \mathbb{E}[r_J T] + 8R\delta,                                        \notag              \\
	            & \leq \frac{1}{n} \sqrt{\mathbb{E}[r_J^2]\mathbb{E}[T^2]} + 8R\delta,                               \notag \\
	            & \leq \frac{1}{2}\mathbb{E}[r_J^2] + \frac{1}{2n^2}\mathbb{E}[T^2] + 8R\delta.\label{eq:mathcal-D}
\end{align}
Here we use Cauchy-Schwarz inequality and the AM-GM inequality.
From the definition of $r_J$ and Eq.~\eqref{eq:4rd}, it holds that
\begin{align}
	\mathbb{E}[r_J^2] & \leq A^2 + \frac{1}{l}\mathbb{E}\qty[\norm{F_J - F^\circ}_2^2]                            \notag \\
	                  & \leq A^2 + \frac{1}{l}\mathbb{E}\qty[\norm{\hat F - F^\circ}_2^2] + 4R\delta.\label{eq:rJ2}
\end{align}
Since $X^{(1)}, \dots, X^{(n)}, \tilde X^{(1)}, \dots, \tilde X^{(n)}$ are independent of each other, we have
\begin{align*}
	\mathbb{V}\qty[\sum_{i=1}^n \frac{g_j(X^{(i)}, \tilde X^{(i)})}{r_j}] & = \sum_{i=1}^n \mathbb{V}\qty[\frac{\norm{F_j(X^{(i)}) - F^\circ( X^{(i)})}_2^2 - \norm{F_j(\tilde X^{(i)}) - F^\circ(\tilde X^{(i)})}_2^2}{r_j}] \\
	                                                                      & \leq \frac{2}{l^2r_j^2}\sum_{i=1}^n \mathbb{E}\qty[\norm{F_j(X^{(i)}) - F^\circ( X^{(i)})}_2^4]                                                   \\
	                                                                      & \leq \frac{8R^2}{lr_j^2}\sum_{i=1}^n \mathbb{E}\qty[\norm{F_j(X^{(i)}) - F^\circ(X^{(i)})}_2^2]                                                   \\
	                                                                      & \leq 8nR^2,                                                                                                                                       \\
	\abs{\frac{g_j(X^{(i)}, \tilde X^{(i)})}{r_j}}                        & \leq \abs{\frac{g_j(X^{(i)}, \tilde X^{(i)})}{r_j}}                                                                                               \\
	                                                                      & \leq \frac{4R^2}{r_j},
\end{align*}
where $\mathbb{V}[\cdot]$ denotes the variance of a random variable.
Using Bernstein's ineqautlity and the union bound, we have, for any $t > 0$,
\begin{align*}
	\Probability(T^2 \geq t) & = \Probability(T\geq \sqrt{t})                                      \\
	                         & \leq 2N\exp{-\frac{t}{2R^2(8n + \frac{4\sqrt{t}}{3r})}}             \\
	                         & \leq 2N\exp{-\frac{t}{32nR^2}} + 2N\exp{-\frac{3r\sqrt{t}}{16R^2}},
\end{align*}
where $r = \min_{j \in [N]}\qty{r_j}$.
Then, for any $t_0 > 0$, we have
\begin{align}
	\mathbb{E}[T^2] & = \int_0^\infty \Probability(T^2 \geq t) \dd{t}                                                                        \notag                    \\
	                & \leq t_0 + \int_{t_0}^\infty \Probability(T^2 \geq t) \dd{t}                                                                \notag               \\
	                & \leq t_0 + 2 N \int_{t_0}^\infty \exp[-\frac{t}{32nR^2}] \dd{t} + 2N \int_{t_0}^\infty \exp[\frac{3r\sqrt{t}}{4R^2}] \dd{t}.\label{eq:integrals}
\end{align}
The integrals in Eq.~\eqref{eq:integrals} can be evaluated as follows:
\begin{align*}
	\int_{t_0}^{\infty}\exp(-\frac{t}{32nR^2})\dd{t} & = \qty[-32R^2 \exp(-\frac{t}{32nR^2})]_{t_0}^\infty                                           \\
	                                                 & = 32nR^2 \exp(-\frac{t_0}{32nR^2}),                                                           \\
	\int_{t_0}^{\infty}\exp(-\frac{3r}{16R^2})\dd{t} & = \qty[-\frac{2(a\sqrt{t} + 1)}{a^2}\exp(-a\sqrt{t})]_{t_0}^\infty\quad(a = \frac{3r}{16R^2}) \\
	                                                 & = \qty(\frac{512R^4}{9r^2} + \frac{32R^2\sqrt{t_0}}{3r})\exp(-\frac{3r\sqrt{t_0}}{16R^2}).
\end{align*}
Let $A = \frac{\sqrt{t_0}}{6n}$. Then, we have $r\geq A = \frac{\sqrt{t_0}}{6n}$ and
\begin{align*}
	\mathbb{E}[T^2] & \leq t_0 + 2N\qty(32nR^2 + 64nR^2 + \frac{2048n^2R^4}{t_0})\exp(-\frac{t_0}{32nR^2}).
\end{align*}
Here, we determine $t_0 = 32nR^2 \log N$. Then, we have
\begin{align}
	\mathbb{E}[T^2] & \leq 32nR^2\qty(\log N + 6 + \frac{4}{\log N}).\label{eq:T2}
\end{align}
Combining~\eqref{eq:mathcal-D},~\eqref{eq:rJ2}, \eqref{eq:T2}, $A^2 = \frac{8R^2\log N}{9n}$, and $\log N \geq 1$, $\mathcal{D}$ can be evaluated as follows:
\begin{align}
	\mathcal{D} & \leq \frac{1}{2}\mathbb{E}[r_J^2] + \frac{1}{2n^2}\mathbb{E}[T^2] + 8R\delta                                                                                               \notag     \\
	            & \leq \frac{1}{2} A^2 + \frac{1}{2} \mathbb{E}\qty[\frac{1}{l}\norm{\hat F - F^\circ}_2^2] + \frac{1}{2n^2}\mathbb{E}[T^2] + 10R\delta                                          \notag \\
	            & \leq \frac{1}{2} A^2 + \frac{1}{2} \mathbb{E}\qty[\frac{1}{l}\norm{\hat F - F^\circ}_2^2] + \frac{16R^2}{n}\qty(\log N + 6 + \frac{4}{\log N}) + 10R\delta            \notag          \\
	            & \leq \frac{1}{2} R(\hat F, F^\circ) + \frac{4R^2}{n}\qty(\frac{37}{9}\log N + 40) + 10R\delta.                                                 \label{eq:mathcal-D-final}
\end{align}

Next, we evaluate $\hat R$. Since $\hat F$ is an empirical risk minimizer, it holds that
\begin{align*}
	\mathbb{E}\qty[\frac{1}{nl} \sum_{i=1}^{n} \norm{\hat F(X^{(i)}) - Y^{(i)}}_2^2] & \leq \mathbb{E}\qty[\frac{1}{nl}\sum_{i=1}^{n}\norm{F(X^{(i)}) - Y^{(i)}}_2^2],
\end{align*}
for any $F \in \mathcal{F}$.
Substituting $Y^{(i)} = F^\circ(X^{(i)}) + \xi^{(i)}$, we have
\begin{align*}
	0 & \leq \mathbb{E}\qty[\frac{1}{nl} \sum_{i=1}^{n} \norm{F(X^{(i)}) - Y^{(i)}}_2^2] - \mathbb{E}\qty[\frac{1}{nl}\sum_{i=1}^{n}\norm{\hat F(X^{(i)}) - Y^{(i)}}_2^2]                                                            \\
	  & = \mathbb{E}\qty[\frac{1}{l}\norm{F(X^{(i)}) - F^\circ(X^{(i)})}_2^2] - \mathbb{E}\qty[\frac{1}{l}\norm{\hat F(X^{(i)}) - F^\circ(X^{(i)})}_2^2] + \frac{2}{nl} \sum_{i=1}^n \mathbb{E}\qty[\ev{\xi^{(i)}, \hat F(X^{(i)})}] \\
	  & = \frac{1}{l}\norm{F - F^\circ}_2^2 + \frac{2}{nl} \sum_{i=1}^n \mathbb{E}\qty[\ev{\xi^{(i)}, \hat F(X^{(i)})}] - \hat R.
\end{align*}
Therefore, we have
\begin{align*}
	\hat R & \leq \frac{1}{l}\norm{F-F^\circ}_2^2 + \frac{2}{nl}\sum_{i=1}^{n} \mathbb{E}\qty[\ev{\xi{(i)}, \hat F(X^{(i)})}].
\end{align*}
For the second term, we have
\begin{align*}
	\frac{2}{nl}\mathbb{E}\qty[\sum_{i=1}^{n}\ev{\xi^{(i)}, \hat F(X^{(i)})}] & = \frac{2}{nl}\mathbb{E}\qty[\sum_{i=1}^{n}\ev{\xi^{(i)}, \hat F(X^{(i)}) - F^\circ(X^{(i)})}]                                                                                          \\
	                                                                          & = \frac{2}{nl}\mathbb{E}\qty[\sum_{i=1}^{n}\ev{\xi^{(i)}, \hat F(X^{(i)}) - F_J(X^{(i)})}] + \frac{2}{nl}\mathbb{E}\qty[\sum_{i=1}^{n}\ev{\xi^{(i)}, F_J(X^{(i)}) - F^\circ(X^{(i)})}].
\end{align*}
By the Cauchy-Schwartz ineqaulity, we have
\begin{align*}
	\frac{2}{nl}\mathbb{E}\qty[\sum_{i=1}^{n}\ev{\xi^{(i)}, \hat F(X^{(i)}) - F_J(X^{(i)})}] & \leq \frac{2}{nl}\mathbb{E}\qty[\qty(\sum_{i=1}^{n}\norm{\xi^{(i)}}_2^2)^{1/2} \qty(\sum_{i=1}^{n} \norm{\hat F(X^{(i)}) - F_J(X^{(i)})}_2^2)^{1/2}] \\
	                                                                                         & \leq \frac{2\delta}{\qty(nl)^{1/2}}\mathbb{E}\qty[\qty(\sum_{i=1}^{n}\norm{\xi^{(i)}}_2^2)^{1/2}]                                                    \\
	                                                                                         & \leq \frac{2\delta}{\qty(nl)^{1/2}}\mathbb{E}\qty[\sum_{i=1}^{n}\norm{\xi^{(i)}}_2^2] ^{1/2}                                                         \\
	                                                                                         & = 2\delta \sigma.
\end{align*}
Define random variables $\varepsilon_1, \dots, \varepsilon_N$ as
\begin{align*}
	\varepsilon_j & := \frac{\sum_{i=1}^{n}\ev{\xi^{(i)}, F_j(X^{(i)}) - F^\circ(X^{(i)})}}{\qty(\sum_{i=1}^{n} \norm{F_j(X^{(i)} - F^\circ(X^{(i)}))}_2^2)^{1/2}}.
\end{align*}
If the denominator is zero, we define $\varepsilon_j = 0$.
Then, we have
\begin{align*}
	\frac{2}{nl}\abs{\mathbb{E}\qty[\sum_{i=1}^{n}\ev{\xi^{(i)}, F_J(X^{(i)}) - F^\circ(X^{(i)})}]} & = \frac{2}{nl}\abs{\mathbb{E}\qty[\qty(\sum_{i=1}^{n} \norm{F_J(X_i) - F^\circ(X_i)}_2^2)^{1/2} \varepsilon_J]}                                   \\
	                                                                                                & \leq \frac{2}{\sqrt{n}l}\mathbb{E}\qty[\frac{1}{n} \sum_{i=1}^{n} \norm{F_J(X_i) - F^\circ(X_i)}_2^2]^{1/2} \mathbb{E}\qty[\varepsilon_J^2]^{1/2} \\
	                                                                                                & \leq \frac{2}{\sqrt{n}}\sqrt{\hat R + 4R\delta} \mathbb{E}\qty[\max_j \varepsilon_j^2]^{1/2}                                                      \\
	                                                                                                & \leq \frac{1}{2}(\hat R + 4R\delta) + \frac{2}{n} \mathbb{E}\qty[\max_j \varepsilon_j^2].
\end{align*}
Since each $\varepsilon_j$ follows $N(0, \sigma^2)$ for given $X^{i}$,
by the same argument as Theorem 7.47 in~\citet{lafferty_concentration_2008}, we have
\begin{align*}
	\mathbb{E}[\max_j \varepsilon_j^2] & \leq 4\sigma^2 \log(\sqrt{2}N)\leq 4\sigma^2 (\log N + 1).
\end{align*}
Therefore, it holds that
\begin{align*}
	\hat R & \leq \frac{1}{l} \norm{F - F^\circ}_2^2 + 2\delta \sigma + \frac{1}{2} (4F\delta + \hat R) + \frac{8}{n}\sigma^2 (\log N + 1)
\end{align*}
and then,
\begin{align}
	\hat R & \leq 2\frac{1}{l} \norm{F - F^\circ}_2^2 + 4(R+\sigma)\delta + \frac{16}{n}\sigma^2 (\log N + 1)\label{eq:hat-R}
\end{align}
holds.

Combining Eq.~\eqref{eq:mathcal-D-final} and~\eqref{eq:hat-R}, we have
\begin{align*}
	R(\hat F, F) & \leq \hat R + \mathcal{D}                                                                                                                                                              \\
	             & \leq \frac{2}{l} \norm{F - F^\circ}_2^2 + 4(R + \sigma)\delta + \frac{16}{n} \sigma^2(\log N + 1) + \frac{1}{2}R(\hat F, F) + \frac{4R^2}{n}\qty(\frac{37}{9}\log N + 40) + 10R\delta,
\end{align*}
and thus,
\begin{align*}
	R(\hat F, F) & \leq \frac{4}{l} \norm{F - F^\circ}_2^2 + 8(R + \sigma)\delta + \frac{32}{n} \sigma^2(\log N + 1) + \frac{8R^2}{n}\qty(\frac{37}{9}\log N + 40) + 20R\delta.
\end{align*}
Since $F$ is arbitrary, it holds that
\begin{align*}
	R(\hat F, F) & \leq 4\inf_{F \in \mathcal{F}}\frac{1}{l} \norm{F - F^\circ}_2^2 + 8(R + \sigma)\delta + \frac{32}{n} \sigma^2(\log N + 1) + \frac{8R^2}{n}\qty(\frac{37}{9}\log N + 40) + 20R\delta,
\end{align*}
which completes the proof.\qed

\section{Proof of Theorem~\ref{thm:covering-number}}\label{sec:proof-covering-number}
For a Transformer $F \in \mathcal{T}(M, U, D, H, L, W, B, S, P)$, let $\theta_F$ be a vector of all the parameters of $F$.
Suppose that $F, \tilde F \in \mathcal{T}(M, U, D, H, L, W, B, S, P)$ satisfies $\norm{\theta_F - \theta_{\tilde F}}_{\infty} \leq \delta$ for $\delta > 0$. That is, for any parameter $\theta$ in $F$, the corresponding parameter $\tilde \theta$ in $\tilde F$ satisfies $\abs{\theta - \tilde \theta} \leq \delta$.
Transformer networks $F$ and $\tilde F$ can be expressed in the form:
\begin{align*}
	F(X) = h_{2M} \circ \dots \circ h_1 \circ (EX + P), \\
	\tilde F(X) = \tilde h_{2M} \circ \cdots \circ \tilde h_1 \circ (\tilde EX + P),
\end{align*}
where $h_i, \tilde h_i \in \Psi(L, W, B, S)$ if $i$ is even, and $h_i, \tilde h_i \in \mathcal{A}(U_{(i+1)/2}, D, H, B)$ if $i$ is odd.
For fixed $X \in \domain$, it holds that
\begin{align}
	& \norm{F(X) - \tilde F(X)}_\infty \leq \norm{h_{2M} \circ \dots h_1(EX + P) - h_{2M} \circ \dots h_1(\tilde EX + P)}_\infty   \notag \\
	                                 & \quad + \sum_{m=1}^{2M}\norm{h_{2M} \circ \dots \circ h_m \circ \tilde h_{m-1} \circ \dots \circ \tilde h_1(\tilde EX + P)
		- h_{2M} \circ \dots \circ h_{m+1} \circ \tilde h_m \circ \dots \circ \tilde h_1(\tilde EX + P)}_\infty.\label{eq:divide}
\end{align}

Since $\norm{P}_{\infty}$ is assumed to be less than $B$, we have $\norm{EX + P}_\infty, \norm{\tilde EX + P}_\infty \leq 2BD$.
By applying Lemma~\ref{lem:norm} repeatedly, we have
\begin{align*}
	\norm{h_m \circ \tilde h_{m-1} \circ \dots \circ \tilde h_1 \circ \enc_P(X)}_\infty & \leq (6HDBW)^{2LM} \cdot 2BD \leq (6HDBW)^{3LM}, \\
	\norm{\tilde h_m \circ \dots \circ \tilde h_1 \circ \enc_P(X)}_\infty               & \leq (6HDBW)^{2LM} \cdot 2BD \leq (6HDBW)^{3LM}.
\end{align*}
In addition, Lemma~\ref{lem:lipschitz} yields that
\begin{align}
	\norm{h_m(X) - h_m(X')}_\infty \leq (6HDMW)^{4L + 6LM}  \norm{X-X'}_\infty \leq (6HDMW)^{10LM}  \norm{X-X'}_\infty \label{eq:lipschitz}
\end{align}
for $\norm{X}_\infty, \norm{X'}_\infty \leq (6HDMW)^{3LM}$.
Therefore, for the first term in Eq.~\eqref{eq:divide}, we have
\begin{align*}
	\norm{h_{2M} \circ \dots h_1(EX + P) - h_{2M} \circ \dots h_1(\tilde EX + P)}_\infty & \leq (6HDBW)^{20LM^2} \norm{EX + P - (\tilde EX + P)}_\infty \\
	                                                                                     & \leq (6HDBW)^{20LM^2} D\delta \leq (6HDBW)^{21LM^2}\delta.
\end{align*}
For any $1 \leq m \leq 2M$, we have
\begin{align*}
	\norm{h_{2M} \circ \dots \circ h_m \circ \dots \tilde h_1 \circ \tilde \enc_P - h_{2M} \circ \dots \circ \tilde h_m \circ \dots \tilde h_1 \circ \tilde \enc_P}_\infty & \leq (6HDMW)^{20M^2L}\norm{h_m(Z) - \tilde h_m(Z)}_\infty,
\end{align*}
where $Z = \tilde h_{m-1} \circ \tilde h_1 \circ \tilde \enc_P(X)$.
Since $\norm{Z} \leq (6HDBW)^{3LM}$, Lemma~\ref{lem:diff} implies that
\begin{align*}
	\norm{h_m(Z) - \tilde h_m(Z)} & \leq (6HDBW)^{4L+9LM} \delta \leq (6HDBW)^{13LM}\delta.
\end{align*}
Thus, we have
\begin{align*}
	\norm{h_{2M} \circ \dots \circ h_m \circ \dots \tilde h_1 \circ \tilde \enc_P - h_{2M} \circ
	\dots \circ \tilde h_m \circ \dots \tilde h_1 \circ \tilde \enc_P}_\infty & \leq (6HDMW)^{20M^2L}(6HDBW)^{13LM}\delta \\
	                                                                          & \leq (6HDMW)^{33M^2L}\delta.
\end{align*}
Then, we have
\begin{align*}
	\norm{F(X) - \hat F(X)}_\infty & \leq (6HDBW)^{21LM^2}\delta + 2M(6HDMW)^{33M^2L}\delta \\
	                               & \leq (6HDMW)^{34M^2L}\delta                            \\
\end{align*}
Here, the number of non-zero components in $\theta_F$ is bounded by $M(S + 3HD^2) + D^2$, where $MS$ for FNN layers, $3MHD^2$ for attention layers, and $Dd \leq D^2$ for an embedding layer.
Therefore, if we fix the sparsity pattern, the covering number is bounded by
\begin{align*}
	\qty(\frac{(6HDMW)^{38M^2L}}{\delta})^{M(S + 3HD^2) + D^2}.
\end{align*}
Since the total number of parameters is bounded by $M(L(W^2+W) + 3HD^2)+D^2 \leq 4M(LW^2 + HD^2)$,
the number of configurations of the sparsity pattern is bounded by
\begin{align*}
	\binom{4M(LW^2 + HD^2)}{M(S + 3HD^2) + D^2} & \leq \qty(4M(LW^2 + HD^2))^{M(S + 3HD^2) + D^2}.
\end{align*}
Therefore, the covering number of Transformer networks is bounded by
\begin{align*}
	\qty(4M(LW^2 + HD^2))^{M(S + 3HD^2) + D^2}	\qty(\frac{(6HDMW)^{34M^2L}}{\delta})^{M(S + 3HD^2) + D^2} & \leq \qty(\frac{(6HDMWL)^{36M^2L}}{\delta})^{M(S + 3HD^2) + D^2},
\end{align*}
which completes the proof.\qed

\section{Proof of Theorem~\ref{thm:estimation-error}}\label{sec:proof-estimation}
From Theorem~\ref{thm:approximation}, there exists a Transformer $\mathcal{T}_R(M, U, D, H, L, W, S, B)$
such that $\norm{\hat F_i - F^\circ_i}_{2, P_X} \lesssim 2^{-T}$ for any $i \in \Z$, where
\begin{align*}
	M      & = 1,                                                          \\
	\log U_1 & \sim T,                                                       \\
	D      & \sim T^{1/\alpha},                                            \\
	H      & \sim T^{1/\alpha},                                            \\
	L      & \sim \max\qty{T^{2/\alpha}, T^2},                             \\
	W      & \sim T^{1/\alpha}2^{T/a^\dagger},                             \\
	S      & \sim T^{2/\alpha} \max\qty{T^{2/\alpha}, T^2}2^{T/a^\dagger}, \\
	\log B & \sim \max\qty{T^{1/\alpha}, T},
\end{align*}
because $\norm{F}_\infty \leq R$.
Therefore, the bias of the estimator $\hat F \in \mathcal{T}_R$ can be evaluated as follows:
\begin{align*}
	\inf_{F' \in \mathcal{T}_R} \frac{1}{r-l+1}\sum_{i=l}^r \norm{F'_i - F^\circ_i}_{2, P_X}^2 & \leq 2^{-2T}.
\end{align*}
From Lemma~\ref{thm:covering-number}, the log covering number $\log \mathcal{N}(\mathcal{T}_R, \delta, \norm{\cdot}_{\infty})$ is evaluated as follows:
\begin{align*}
	\log \mathcal{N}(\mathcal{T}_R, \delta, \norm{\cdot}_{\infty}) & \leq \log \mathcal{N}(\mathcal{T}, \delta, \norm{\cdot}_{\infty}) \lesssim 2^{T/a^\dagger} T^{2/\alpha + 1}\max\qty{T^{4/\alpha}, T^4} \log \frac{T}{\delta}.
\end{align*}
Therefore, from Lemma~\ref{thm:estimation-error}, the ERM estimator $\hat F$ satisfies
\begin{align*}
	R_{l, r}(\hat F, F^\circ) & \lesssim 2^{-2T} + \frac{2^{T/a^\dagger}T^{2/\alpha + 1} \max\qty{T^{4/\alpha}, T^4} \log(T/\delta)}{n} + \delta.
\end{align*}
By letting $T = \frac{a^\dagger}{2a^\dagger + 1} \log n$ and $\delta = 1/n$, we have
\begin{align*}
	R_{l, r}(\hat F, F^\circ) & \lesssim n^{-\frac{2a^\dagger}{2a^\dagger + 1}} (\log n)^{2/\alpha + 2} \max\qty{(\log n)^{4/\alpha}, (\log n)^4}.
\end{align*}
\qed

\section{Proof of Theorem~\ref{thm:estimation-piecewise}}\label{sec:proof-estimation-piecewise}
Let 
\begin{align*}
	M      & = T^{1/\alpha},                                                          \\
	\log U_i & \lesssim \max\qty{\log T, \log V},                                                       \\
	D      & \sim T^{2(\beta + 1)/\alpha}\log V,                                            \\
	H      & \sim (\log T)^{1/\alpha},                                            \\
	L      & \sim \max\qty{T^{2/\alpha}, T^2},                             \\
	W      & \sim T^{1/\alpha}2^{T/a^\dagger},                             \\
	S      & \sim T^{2/\alpha} \max\qty{T^{2/\alpha}, T^2}2^{T/a^\dagger}, \\
	\log B & \sim \max\qty{T^{1/\alpha}, T, \log \log V}.
\end{align*}
Then, by the same argument as in Theorem~\ref{thm:estimation-error}, it holds that
\begin{align*}
	\log \mathcal{N}(\mathcal{T}_R, \delta, \norm{\cdot}_\infty) & \lesssim T^{5/\alpha + 1}\max\qty{T^{4/\alpha}, T^4}2^{T/a^\dagger} (\log V)^2 \log\qty(\frac{T\log V}{\delta}).
\end{align*}
By letting $T = \frac{a^\dagger}{2 a^\dagger + 1} \log n$ and $\delta = 1/n$, we have
\begin{align*}
	R_{l, r}(\hat F, F^\circ) & \lesssim n^{-\frac{2a^\dagger}{2a^\dagger + 1}}(\log n)^{5/\alpha + 2}\max\qty{(\log n)^{4/\alpha}, (\log n)^4} (\log V)^3 .
\end{align*}
\qed
\end{document}